%% file: main.tex
\documentclass[twoside]{article}

\usepackage[accepted]{aistats2022}

\usepackage[round,sort&compress]{natbib}

\input{packages}

\definecolor{blue}{HTML}{241fed}
\definecolor{red}{HTML}{ED1C24}

\begin{document}

\twocolumn[

\aistatstitle{Can Pretext-Based Self-Supervised Learning Be Boosted by Downstream Data? A Theoretical Analysis}

\aistatsauthor{  Jiaye Teng* \And Weiran Huang* \And  Haowei He* }

\aistatsaddress{ IIIS, Tsinghua University \And Huawei Noah's Ark Lab \And IIIS, Tsinghua University } ]

\begingroup
\begin{NoHyper}
\renewcommand\thefootnote{*}
\footnotetext{Equal contribution.
Correspondence to Weiran Huang $<$weiran.huang@outlook.com$>$.}
\end{NoHyper}
\endgroup

\input{text/0abstract}
\input{text/1intro}
\input{text/3preliminary}
\input{text/30mainresult}
\input{text/4fails}
\input{text/5experiments}

\input{text/2relatedwork}

\input{text/6conclusion}

\subsubsection*{Acknowledgements}
The authors would like to thank Yang Yuan (Tsinghua University) for his helpful discussion. 
Besides, the authors thank the reviewers in AISTATS~2022 for their careful and constructive comments.

\bibliography{ssl}

\clearpage
\appendix

\onecolumn \makesupplementtitle

\input{supplement}

\end{document}

%% file: packages.tex
\usepackage{caption} 
\usepackage{enumitem}
\usepackage{amsthm}
\usepackage{amsfonts}
\usepackage{amsmath}
\usepackage{bm}
\usepackage{amssymb}
\usepackage{graphicx} 
\usepackage{nccmath}
\usepackage{subfigure}
\usepackage{thmtools}
\usepackage{thm-restate}
\usepackage{hyperref}       %
\hypersetup{colorlinks=true, linkcolor=red, citecolor=blue,urlcolor=black}
\usepackage{xcolor}         %
\usepackage{url}
\usepackage{cleveref}

\theoremstyle{plain}
\newtheorem{definition}{Definition}

\theoremstyle{definition}
\newtheorem{remark}{Remark}

\usepackage{soul}

\newcommand{\ci}{\!\perp\!}
\newcommand{\mycolon}{\!:\!}
\newcommand{\midbar}{{\, | \,}}

\DeclareMathOperator*{\argmin}{arg\,min}

\newcommand{\R}{\mathbb{R}}
\DeclareMathOperator*{\E}{\mathbb{E}}
\DeclareMathOperator*{\Cov}{\mathrm{Cov}}
\newcommand{\norm}[1]{\left\|#1\right\|}

\newcommand{\cX}{\mathcal{X}}
\newcommand{\cY}{\mathcal{Y}}
\newcommand{\cZ}{\mathcal{Z}}
\newcommand{\cH}{\mathcal{H}}
\newcommand{\cF}{\mathcal{F}}

\renewcommand{\O}{\mathcal{O}}

\newcommand{\x}{\boldsymbol{x}}
\newcommand{\z}{\boldsymbol{z}}
\newcommand{\y}{\boldsymbol{y}}
\newcommand{\xSample}{\boldsymbol{x}_{i}}

\newcommand{\ySample}{\boldsymbol{y}_{i}}
\newcommand{\xMatrix}{X}

\newcommand{\yMatrix}{Y}

\newcommand{\bM}{B}

\newcommand{\cFuntion}{f}
\newcommand{\cFuntionStar}{{f^{*}}}
\newcommand{\hatcFuntion}{\hat{f}}

\newcommand{\PopulationLoss}{\mathcal{L}}

\newcommand{\dimx}{{d_x}}
\newcommand{\dimz}{{d_z}}
\newcommand{\dimy}{{d_y}}
\newcommand{\dimc}{{d_f}}
\newcommand{\samples}{n}

\newcommand{\tr}{\operatorname{tr}}

\usepackage{nccmath}

%% file: text/0abstract.tex
\begin{abstract}
Pretext-based self-supervised learning learns the semantic representation via a handcrafted pretext task over unlabeled data and then uses the learned representation for downstream tasks, which effectively reduces the sample complexity of downstream tasks under Conditional Independence (CI) condition~\citep{lee2020predicting}.
However, the downstream sample complexity gets much worse if the CI condition does not hold.
One interesting question is whether we can make the CI condition hold by using downstream data to refine the unlabeled data to boost self-supervised learning.
At first glance, one might think that seeing downstream data in advance would always boost the downstream performance.
However,
we show that it is not intuitively true and point out that in some cases, it hurts the final performance instead.
In particular, we prove both model-free and model-dependent lower bounds of the number of downstream samples used for data refinement.
Moreover, we conduct various experiments on both synthetic and real-world datasets to verify our theoretical results.

\end{abstract} 

%% file: text/1intro.tex
\newcommand{\procedure}{Processor-based SSL (PSSL) \ }

\section{INTRODUCTION}
Data representations used to be learned in a supervised or semi-supervised learning way, e.g., \citet{he2016deep,lee2013pseudo}.
Recently, self-supervised learning has drawn massive attention for its fantastic data efficiency and generalization ability, with many state-of-the-art models following this paradigm in computer vision \citep{chen2020simple,he2020momentum}, language modeling \citep{devlin2018bert,radford2018improving}, graph learning \citep{peng2020self}, etc.
It learns data representations through self-supervised tasks 
and then uses the learned representations for downstream prediction tasks.
Such a paradigm involves many unlabeled data that are easier to access to reduce the sample complexity of labeled data.

The renaissance of self-supervised learning began with artificially designed {\em pretext tasks} as self-supervised tasks, such as image colorization \citep{zhang2016colorful}, inpainting \citep{pathak2016context}, solving jigsaw puzzles \citep{noroozi2016unsupervised}, and predicting image rotations \citep{gidaris2018unsupervised} in computer vision, and predicting next word \citep{radford2018improving} in natural language processing.
Meanwhile, empirical evidence reveals that the representations learned from inappropriate pretext tasks may fail to transfer to downstream tasks \citep{yamaguchi2019multiple,zoph2020rethinking}.
Therefore, how pretext tasks affect the performance of downstream tasks is crucial to pretext-based self-supervised learning.

There are some empirical studies of the above question (e.g., \citet{metzger2020evaluating}), 
however, the theoretical study is still at an early stage.
A recent work~\citep{lee2020predicting} proves that, under conditional independence between the components of the pretext task conditional on the downstream label,
the downstream task achieves minimal sample complexity.
In particular, it shows that for Gaussian variables, the sample complexity of downstream tasks is reduced to $\Tilde{\O}(\text{dim}(\y))$\footnote{We use notation $\Tilde{\O}$ to hide log factors in this paper, and abbreviation ``$\text{dim}$'' stands for dimension.} 
when Conditional Independence (CI) $\x \ci \z \midbar \y $ holds, where $\x$, $\z$, $\y$ denotes the input variable, pretext label, and downstream label, respectively. 
As a comparison, the sample complexity of directly using $\x$ to predict $\y$ is $\tilde\O(\text{dim}(\x))$, where the dimension of $\x$ is supposed to be much larger than the dimension of~$\y$.
However, if the CI condition $\x \ci \z \midbar \y $ does not hold, the downstream sample complexity increases to $\Tilde{\O}(\text{dim}(\z))$.
This is because there is redundant mutual information between $\x$ and $\z$ which is unrelated to $\y$. %

In practice, the CI condition rarely holds, and thus self-supervised learning cannot realize its full potential.
An interesting question raises: 
\begin{center}\emph{Can we make the CI condition hold with the help of downstream data to boost self-supervised learning?}
\end{center}
To formulate the above question, we
introduce a learnable function~$\cFuntion$ (called data processor) to refine the pretext data such that  $\cFuntion(\x) \ci \z \midbar \y$ holds, and the downstream data (or extra data from downstream data distribution) are allowed to be accessed to learn the processor.
If such processor~$\cFuntion$ exists, 
according to the result of \citet{lee2020predicting},
the downstream sample complexity is reduced by
replacing all the unlabeled data $\x$ with $\cFuntion(\x)$.
Otherwise, self-supervised learning cannot be boosted by the downstream data (w.r.t. the order of sample complexity).

At first glance, one might think that seeing downstream data in advance would always boost downstream task performance.
However, we show that the above intuition is not always true and point out that the downstream performance will be hurt in some cases.
Therefore, our results validate self-supervised learning in some sense since we prove that it is better not to use downstream data in such cases, as standard self-supervised learning does.

Specifically, we first rigorously formulate the necessary conditions that data processor $f$ needs to satisfy in Section~\ref{sec:preliminary}.
One is a relaxation of condition $\cFuntion(\x) \ci \z \midbar \y$, and the other condition ensures that $f(\x)$ has enough ability to predict $\y$.
We then design an ingenious loss function to learn such function $f$ in Section~\ref{sec: coarse}.
We demonstrate the rationality of the loss by proving that any function $f$ minimizing the proposed loss satisfies the above criteria. 
Based on the proposed loss, we theoretically study the number of downstream data required.
We derive a model-free lower bound in Section~\ref{sec:warm-up} and show that involving limited downstream data provably hurts model performance in the representation learning process.
Furthermore, we consider the model capacity of function class and
provide a model-dependent lower bound in Section~\ref{sec:model-based result}.
The lower bound gets higher when the model capacity gets larger, indicating that the number of downstream samples needs to match the model capacity.
To verify our theoretical results, we conduct experiments on both synthetic and real-world datasets in Section~\ref{sec:experiments}.
Due to space limitations, all proofs
are deferred to the supplementary material.

In summary, our contributions are listed as follows:
\begin{itemize}[leftmargin=20pt]
\item \emph{Novelty}: We rigorously prove that learning the processor $f$ with insufficient downstream data fails to boost self-supervised learning, which contradicts the intuitive wisdom and validates the standard self-supervised learning framework.

\item \emph{Technical Contributions}: We provide both model-free and model-dependent lower bounds for the sample size required in the training processor. Moreover, we show that more complex model classes are more likely to lead to failure.

\item \emph{Experiments}: We conduct several experiments on both synthetic datasets and real-world datasets to verify our theoretical results.
We observe that 
training the processor with insufficient downstream data significantly decreases downstream task performance. 

\end{itemize}

%% file: text/3preliminary.tex
\newcommand{\ext}{n_0}

\section{NOTATIONS AND PROBLEM FORMULATION}
\label{sec:preliminary}

\textbf{Data Distribution.}
Let $\x \in\cX\subset \R^{\dimx}$, $\z\in\cZ\subset \R^{\dimz}$, $\y\in\cY\subset \R^{\dimy}$ denote the input variable, pretext label, and downstream label, respectively,
where the data are sampled from a joint distribution $(\x, \y, \z) \sim \mathcal{P}$ and $\dimx \ge \dimz \ge \dimy$.
The pretext dataset with $n_1$ samples $( X_{pre}, Z_{pre} ) \in \R^{\dimx\!\times\!\!\:n_1} \times\R^{\dimz\!\times\!\!\:n_1}$ and the downstream dataset with $n_2$ samples $( X_{down}, Y_{down} )\in\R^{\dimx\!\times \!\!\:n_2}\times\R^{\dimy\!\times\!\!\:n_2}$ are i.i.d. drawn from the joint distribution of $(\x, \z)$ and $(\x, \y)$, respectively.
We assume $\mathbb{E} [\x] = \mathbb{E} [\z] =0$ without loss of generality and denote $\Cov[\x, \z | \y] \triangleq \mathbb{E}[\x \z^\top ] - \mathbb{E}[\x\y^\top ] (\mathbb{E}[\y \y^\top])^{-1} \mathbb{E}[\y \z^\top]$ the partial covariance matrix.
Unless stated otherwise, we consider the $L_2$-loss function in the following content\footnote{When subscript is omitted, notation~$\norm{\cdot}$ stands for $L_2$-norm or Frobenius norm for vectors and matrices.}.

\textbf{Pretext-based Self-Supervised Learning.}
In this paper, we follow the formulation of pretext-based self-supervised learning proposed in \citet{lee2020predicting}.
The algorithm is roughly split into two steps:\\
\raisebox{1pt}{$\;\scriptstyle\bullet\;$}
 \emph{Step 1 (pretext):} Learn representation $\widehat\psi$ from  pretext task samples $(X_{pre}, Z_{pre})$.\\
\raisebox{1pt}{$\;\scriptstyle\bullet\;$}
\emph{Step 2 (downstream):} Perform linear regression of $Y_{down}$ on $\widehat\psi(X_{down})$ and get $\widehat W$.

We next introduce the two steps in detail.
We first learn a representation $\widehat\psi \colon \R^{\dimx} \rightarrow \R^{\dimz}$ by minimizing the empirical risk\footnote{When the context is clear, we abuse the notation $\hat{\psi}\colon \R^{d_x \times n} \rightarrow \R^{d_z \times n}$ over a dataset to denote $\hat{\psi}\colon \R^{\dimx} \rightarrow \R^{\dimz}$ over each data point in the dataset.}:
 \begin{align*}
        \widehat\psi=\argmin_{\psi\in \cH}  \frac{1}{n_1}\norm{ \psi(X_{pre})-Z_{pre}}^2,
\end{align*}
where $\cH$ denotes the candidate function class (e.g., the class of multi-layer neural networks).
Then a linear layer $\widehat{W}\in{\R^{\dimy\! \times\!\!\; \dimz}}$ following the representation $\widehat\psi(\cdot)$ is learned via the downstream task, i.e.,
\begin{align*}
    \widehat W = \argmin_{W \in \R^{\dimy \times \dimz}} \frac{1}{n_2} \norm {W \widehat \psi(X_{down})-Y_{down}}^2,
\end{align*}
    and the final predictor for the downstream task is $\hat g(\cdot)=\widehat W \widehat\psi(\cdot)$.

\textbf{Conditional Independence.}
\citet{lee2020predicting} point out that Conditional Independence (CI) condition $\x \ci \z \midbar \y$ plays crucial roles in self-supervised learning, which largely reduce the downstream sample complexity. 
However, without the guarantee of CI condition, the downstream sample complexity gets worse because there is much redundant information (irrelevant to the downstream task) involved in the representation during the learning via the pretext task. 
Therefore, for the downstream task, the mapping from the redundant features to the downstream label also needs to be learned, leading to a larger sample complexity.

The above discussion inspires us to introduce 
a learnable data processor $f$, which refines the unlabeled data such that CI condition $f(\x) \ci \z \mid \y$ holds.
The processor $f$ can help reduce the downstream sample complexity by simply replacing all the unlabeled data $X$ with $f(X)$ in self-supervised learning.
Formally, the following two criteria are essential for a meaningful processor\footnote{Assume that $\mathbb{E}[\cFuntion(\x)] = 0$ for simplicity.} $f$:
\begin{gather}
   \Cov[\cFuntion(\x), \z \midbar \y]= 0  \tag{C1}\label{C1},\\
  f \in \argmin_f  \E\norm{\y- W_{\y,f(\x)}^* \cFuntion(\x) }^2 , \tag{C2}\label{C2}
\end{gather}
where $W_{\y,f(\x)}^*\triangleq\argmin_{W\in \R^{\dimy\! \times\!\!\; \dimc}} \E\norm{\y-W\cFuntion(\x)}^2$ 
is defined as the best linear predictor of $\y$ on $f(\x)$. 

Criterion~\eqref{C1} guarantees that the processor $\cFuntion(\x)$ removes the redundant information which is useless for predicting $\y$.
In fact, Criterion~\eqref{C1} is a Conditional Uncorrelatedness (CU) condition, which is a relaxation of the CI condition. 
We use CU instead of CI because CU is weaker than CI when proving a negative result.
In addition, setting $\Cov[\cFuntion(\x), \z \midbar \y]=0$ as the only objective may lead to non-informative processors.
For example, if $\cFuntion(\x)$ is an independently random noise, Criterion~\eqref{C1} naturally holds, but such $\cFuntion(\x)$ does not carry any information of $\y$ and thus cannot be used to predict $\y$.
Therefore, we need Criterion~\eqref{C2} to ensure that applying function $\cFuntion$ to input variable $\x$ maintains the information for predicting $\y$.

To summarize, if processor $f$ fails to satisfy Criterion~\eqref{C1}, the downstream sample complexity cannot be improved according to the results in \citet{lee2020predicting}. 
If processor $f$ fails to satisfy Criterion~\eqref{C2}, $f(\x)$ will not contain all the information for predicting $\y$, leading to a at least constant generalization error. 

Can we learn a processor $f$ which satisfies both Criterion~\eqref{C1} and \eqref{C2}?
To answer it, 
we first propose a provably rational loss for learning $f$ in Section~\ref{sec: coarse}.
Then we show that a training processor with insufficient downstream data provably fails to get a processor satisfying the above criteria simultaneously in some cases in Section~\ref{sec:main result}.
Moreover, we conduct experiments and observe the existence of the above phenomenon shown in Section~\ref{sec:experiments}.

%% file: text/30mainresult.tex
\newcommand{\Xss}{{X_{down1}}}
\newcommand{\Yss}{{Y_{down1}}}
\newcommand{\Xssdown}{{X_{down2}}}
\newcommand{\Yssdown}{{Y_{down2}}}

\section{LEARNING PROCESSOR \texorpdfstring{$\bm f$}{f}}\label{sec: coarse}

In this section, we focus on learning the processor $f$ satisfying both
Criterion~\eqref{C1} and \eqref{C2}.
We first propose a carefully designed loss in Section~\ref{sec:loss} and then give proof of its rationality in Section~\ref{sec:rationality}.

\subsection{Loss Design}\label{sec:loss}
To learn the processor $f$, one needs to obtain a portion of downstream data before the pretext tasks\footnote{{We note that there is a different setting that provides extra data before the pretext tasks. Similar techniques can be directly applied into the extra data settings.}} since Criterion~\eqref{C1} is related to the downstream label $\y$.
We split the downstream dataset $(X_{down}, Y_{down})$ to two folds, one with $n_0$ samples for processor training $(\Xss, \Yss)$ and the other with $n_2 - n_0$ samples for the original downstream tasks $(\Xssdown, \Yssdown)$.

\textbf{Criterion~\eqref{C1}.}
Recall that Criterion~\eqref{C1} requires that $\cFuntion(\x)$ and $\z$ are uncorrelated conditional on $\y$. 
This indicates that $f(\x)$ should have less ability to predict $\z$ under linear regimes,
which can be measured by loss\footnote{We use a linear layer following the processor $f(\x)$ as the representation $\widehat \psi(\x)$ for simplicity since processor $\cFuntion$ is allowed to be non-linear.}
\begin{align*}
    \PopulationLoss_1=- \E_{\x,\z}\norm{ \z - W^*_{\z, \cFuntion(\x)} \cFuntion(\x)}^2,
\end{align*}
where $W^*$ is defined as the best linear projection from the second subscript to the first subscript, i.e.,
\begin{align*}
W^*_{{\bm b}, {\bm a}} &\triangleq \argmin_{W}  \E_{{\bm b}, {\bm a}} \norm{{\bm b} - W {\bm a}}^2.
\end{align*}
If $f(\x)$ can fit $\z$ well, then $\PopulationLoss_1$ will be large.
Thus, minimizing $\PopulationLoss_1$ tends to pull $\z$ and $f(\x)$ apart.

\textbf{Criterion~\eqref{C2}.}
For Criterion~\eqref{C2}, we need to guarantee that applying $f$ to $\x$ does not lose the information for predicting $\y$, indicating that $f(\x)$ can still fit $\y$ well.
It can be formulated as
\begin{align*}
    \PopulationLoss_2= \E_{\x,\y}\norm{ \y - W^*_{\y, \cFuntion(\x)} \cFuntion(\x)}^2.
\end{align*}
The smaller $\PopulationLoss_2$ indicates that $f(\x)$ contains more information related to $\y$.

To minimize $\PopulationLoss_1$ and $\PopulationLoss_2$ simultaneously, we define the total population loss as $\lambda\PopulationLoss_1+\PopulationLoss_2$, where $\lambda>0$ is the penalty coefficient, i.e.,
for any data distribution $\mathcal{P}$,
\begin{equation}\label{eqn: lossterm}
\begin{split}
\PopulationLoss(\cFuntion; \mathcal{P}) \triangleq  \E_{(\x, \z, \y) \sim \mathcal{P}} &\left[\left\| \y  - W^*_{\y, \cFuntion(\x)} \cFuntion(\x)  \right\|^2 \right. \\
&\left.  - \lambda\! \left\| \z - W^*_{\z, \cFuntion(\x)} \cFuntion(\x)  \right\|^2 \right].
\end{split}
\end{equation}

One might wonder that forcing $f(\x)$ to fit $\y$ but not to fit $\z$ simultaneously seems to be in conflict with each other since $\z$ contains all the information of $\y$.
In Section~\ref{sec:rationality}, we will give strict proof of its rationality.
The intuition of such loss design is to keep the useful information for predicting $\y$, as well as remove the redundancy information unrelated to $\y$.

The corresponding training loss of $\PopulationLoss(\cFuntion; \mathcal{P})$
can be defined as\footnote{We again abuse the notation $f \colon \R^{\dimx \times n} \rightarrow \R^{\dimc \times n}$ over a dataset to denote $f \colon \R^{\dimx} \rightarrow \R^{\dimc}$ over each data point in the dataset.}:
\begin{equation}\label{eqn: training loss}
\begin{split}
\PopulationLoss_{n_1, \ext}(\cFuntion; \mathcal{P}) \triangleq  & \frac{1}{\ext}  \left\| \Yss - \widetilde{W}_2 \cFuntion(\Xss)  \right\|^2 \\
&-  \frac{\lambda}{n_1}  \left\| Z_{pre} - \widetilde{W}_1  \cFuntion(X_{pre}) \right\|^2,
\end{split}
\end{equation}
where $\widetilde{W}_{1}$, $\widetilde{W}_{2}$ represent respectively the best empirical linear predictors for the pretext data and the downstream data, and $n_1$, $\ext$ are respectively the sample sizes of pretext data $(X_{pre},Z_{pre})$ and the downstream data used in process training $(\Xss, \Yss)$.

\subsection{Rationality of Loss}\label{sec:rationality}

We provide the following theorem to show the rationality of loss $\PopulationLoss(f;\mathcal{P})$, which states that by choosing a proper penalty coefficient $\lambda$, there exist numerous distributions such that the processor $\cFuntion$ which minimizes the population loss $\PopulationLoss(f;\mathcal{P})$ satisfies Criterion \eqref{C1} and \eqref{C2}.

\begin{restatable}[Rationality of Loss]{thm}{Rationality}
\label{thm: rationality}
Assume that there exists a ground truth $\cFuntionStar \in \mathcal{F}$ satisfying both Criterion~\eqref{C1} and \eqref{C2}, where $\mathcal{F}$ is the candidate function class.
Assume that for each $\cFuntion \in \mathcal{F}$, the matrix $\E[\cFuntion(\x)\cFuntion^\top(\x)]$ is nonsingular.
Let $\mathcal{A}_{\mathcal{P}}$ be the set of processors that minimize the population loss under data distribution $\mathcal{P}$:
\begin{equation*}
    \mathcal{A}_{\mathcal{P}} = \left\{\cFuntion: \cFuntion \in \argmin_\cFuntion \mathcal{L}(\cFuntion; {\mathcal{P}})  \right\}.
\end{equation*}

Let $\mathcal{B}_{\mathcal{P}}$ be the set of processor that satisfy the criteria, namely:
\begin{equation*}
    \mathcal{B}_{\mathcal{P}} = \left\{ \cFuntion: \cFuntion\ \text{satisfies Criterion~(\ref{C1}) and (\ref{C2})} \right\}.
\end{equation*}

Then by choosing a proper parameter $\lambda$, there exist a number of population distributions \{$\mathcal{P}$\}'s such that every function in $\mathcal{A}_{\mathcal{P}}$ satisfies Criterion~(\ref{C1}) and Criterion~(\ref{C2}), namely, the following defined set $\mathbb{S}$ is not empty: 
\begin{equation*}
    \mathbb{S} \triangleq \left\{ \mathcal{P}: \mathcal{A}_{\mathcal{P}} \subset  \mathcal{B}_{\mathcal{P}} \right\} \neq \phi.
\end{equation*}
\end{restatable}

In Theorem~\ref{thm: rationality}, we show that there \emph{exist} numerous data distributions such that under those distributions, Criterion~\eqref{C1} and \eqref{C2} hold when the population loss is minimized.  
However, we will show in the next section that for those data distributions, minimizing the corresponding empirical loss fails to satisfy the criteria when training with insufficient downstream samples.
Therefore, when we do not have much knowledge about the data distribution,
it is better not to use downstream data, as standard self-supervised learning does.
Otherwise, we may get a worse downstream task performance.

%% file: text/4fails.tex
\newcommand{\lin}{V}
\newcommand{\gOne}{{g_1}}
\newcommand{\gTwo}{{g_2}}
\newcommand{\rhoOne}{{\rho_1}}
\newcommand{\rhoTwo}{{\rho_2}}

\section{INSUFFICIENT DOWNSTREAM SAMPLES PROVABLY FAILS}\label{sec:main result}
Theorem~\ref{thm: rationality} has shown that we can learn the processor $f$ by minimizing the population loss, such that $f$ satisfies both Criterion \eqref{C1} and \eqref{C2}.
However, we can only access a finite number of downstream samples in practice.
In this section, we will prove that even under the rational loss, one fails to learn a processor that satisfies both Criterion~\eqref{C1} and \eqref{C2} with limited downstream data when training the processor.
Specifically, we first provide a model-free lower bound in Section~\ref{sec:warm-up} as a warm-up and extend the loss to a general loss in Section~\ref{sec:model-free}. 
We further take the structure of function class into account and give a model-dependent lower bound in Section~\ref{sec:model-based result}.

\subsection{Warm-Up: Model-Free Results}\label{sec:warm-up}

To better understand the role of the downstream samples, we consider the following loss (See Eq.~\eqref{Eqn:inftyUnlabeled}), which differs from the population loss only in the downstream part (i.e., infinite pretext samples).
Concretely, we replace the population loss of the downstream part with its empirical version.

\begin{equation}\label{Eqn:inftyUnlabeled}
\begin{split}
 \PopulationLoss_{\infty,\ext}(f,\mathcal{P}) \triangleq  &\frac{1}{\ext}  \left\| \Yss - \widetilde{W}_2 \cFuntion(\Xss)  \right\|^2  \\
&- \lambda \E_{\x,\z} \left\| \z - W^*_{\z, \cFuntion(\x)} \cFuntion(\x)  \right\|^2.
\end{split}
\end{equation}

We show that minimizing the above empirical loss fails to satisfy Criterion~\eqref{C1} and \eqref{C2} simultaneously when $\ext = o(\dimc)$, where $\dimc$ is the dimension of $f(\x)$.

\begin{restatable}[Model-Free Failure]{thm}{warmup}
\label{thm: warm-up}
Assume that there exists a ground truth $\cFuntionStar \in \mathcal{F}$ satisfying both Criterion~\eqref{C1} and \eqref{C2}, where $\mathcal{F}$ is the candidate function class.
And assume that  there exist function $f_1 \in \mathcal{F}$ and $f_2 \in \mathcal{F}$ such that covariance $\Cov[f_1(\x),\y]\neq 0$ and $\Cov[f_2(\x),\z]=0$.
Assume that for any $f_2 \in \mathcal{F}$, the matrix $\E[\cFuntion(\x)\cFuntion^\top(\x)]$ is nonsingular.
Let $ \mathcal{A}^\prime_{\mathcal{P}}$ be the set of processor that minimizes the training loss with $n_0 = o(\dimc)$ downstream samples and infinite pretext samples, where $\dimc$ is the dimension of $f(\x)$,
\begin{equation*}
    \mathcal{A}^\prime_{\mathcal{P}} = \left\{\cFuntion: \cFuntion \in \argmin_\cFuntion \PopulationLoss_{\infty,\ext}(\cFuntion; {\mathcal{P}})  \right\}.
\end{equation*}
And let $\mathcal{B}_{\mathcal{P}}, \mathbb{S}$ be defined as in Theorem~\ref{thm: rationality}.
Then there \emph{exists} a distribution $\mathcal{P}^0 \in \mathbb{S}$ (which means $ \mathcal{A}_{\mathcal{P}^0 }\subset  \mathcal{B}_{\mathcal{P}^0} $), such that all the function that minimizes the loss $\PopulationLoss_{\infty,\ext}$ cannot meet Criterion~\eqref{C1} and \eqref{C2} simultaneously, namely:
$$ \mathcal{A}^\prime_{\mathcal{P}^0} \cap  \mathcal{B}_{\mathcal{P}^0} = \phi.$$
\end{restatable}

In Theorem~\ref{thm: warm-up}, the condition $\Cov[f_1(\x),\y]\neq 0$ assumes that the function class of $f$ is meaningful such that $f$ has ability to predict $y$, while
$\Cov[f_2(\x),\z]=0$ could be easily satisfied when the function class is chosen independently with the pretext task.
Theorem~\ref{thm: warm-up} indicates that under mild assumptions,
to train a function $f$ satisfying both Criterion~\eqref{C1} and \eqref{C2}, we need at least $\Omega(\dimc)$ labeled downstream samples, although unlimited unlabeled data can be accessed.
One can expect the larger downstream sample size requirement with finite unlabeled data.

\begin{remark}
When $n_0=o(d_f)$, even if we have infinite pretext data $(X_{pre},Z_{pre})$, the criteria cannot be satisfied,
leading to a constant generalization error.
This means that
the downstream performance gets worse, even if we use infinite data $(X_{down2},Y_{down2})$ for downstream training.
Therefore, one can conclude that the failure is due to the lack of downstream data used in processor training $(\Xss, \Yss)$.
\end{remark}

\subsection{Model-Free Results for General Loss}\label{sec:model-free}
Theorem~\ref{thm: warm-up} is derived based on the $L_2$ loss that we design in Section~\ref{sec:loss}. 
One may concern the generality.
In fact, we can extend the loss in Theorem~\ref{thm: warm-up} to a more general one to strengthen our result and show that the failure of limited downstream samples is not caused by a specific loss. 
Inspired by the previously designed $L_2$ loss, we define the \emph{general loss} as
\begin{equation*}
\begin{split}
    \overline{\PopulationLoss}\left(\cFuntion,\lambda;\gTwo,  \rhoTwo,\gOne, \rhoOne\right) =&
   \E_{\x,\y} \gTwo(\rhoTwo(\y, W^*_{\y, f(\x)}\cFuntion(\x))) \\
   &- \lambda \E_{\x,\z} \gOne(\rhoOne(\z, W^*_{\z, f(\x)} \cFuntion(\x))).
\end{split}
\end{equation*}
where $\gOne$ and $\gTwo$ are strictly increasing functions over $[0, \infty)$, and $\lambda>0$ denotes a positive penalty coefficient. 
Mappings $\rhoOne\colon \R^{\dimz} \times \R^{\dimz} \to \R_{\ge 0}$ and $\rhoTwo\colon \R^{\dimy} \times \R^{\dimy} \to \R_{\ge 0}$ are distance metrics,
and $\rhoOne$ is consistent with $\rhoTwo$ in subspace $\R^{\dimy} \times \R^{\dimy}$, i.e.,
for all $a, b \in \R^\dimy$,
$\rhoOne((a, \vec{0}_{\dimz-\dimy}), (b, \vec{0}_{\dimz-\dimy})) = \rhoTwo(a, b)$ holds.
We abuse $W^*_{\y, f(\x)}$ and $W^*_{\z, f(\x)}$ as the best linear predictors which minimize the corresponding loss $\E \gTwo(\rhoTwo(\y, W\cFuntion(\x)))$ and $\E \gOne(\rhoOne(\z, W \cFuntion(\x)))$, respectively.

It is not hard to verify that
the loss $\PopulationLoss(f,\mathcal{P})$ defined in Equation~\eqref{eqn: lossterm} is a special case of the above general loss, by setting $\rhoTwo, \rhoOne$ to be Euclidean distances and $\gTwo, \gOne$ to be the square functions.
The corresponding training loss with infinite pretext data is
\begin{equation*}
\begin{split}
    \overline{\mathcal{L}}_{\infty,\ext}(\cFuntion; \mathcal{P})&=   \frac{1}{\ext} \gTwo\left(\rhoTwo\left(\Yss , \widetilde{W}_1\cFuntion (\Xss)\right)\right)\\
    &- \lambda  \E_{\x,\z} \gOne\left(\rhoOne\left(\z, W_{\z,f(\x)}^*\cFuntion(\x)\right)\right),
\end{split}
\end{equation*}
where $\widetilde{W}_1$ is the best empirical linear predictor for the downstream data\footnote{We again abuse the notation $g(\rho_2)\colon \R^{\dimy\times n } \times \R^{\dimy \times n } \rightarrow \R$ on a dataset to denote the summation over $g(\rho_2)\colon \R^{\dimy} \times \R^{\dimy} \rightarrow \R$ on each point in the dataset.}.
We next prove that the procedure provably fails in Theorem~\ref{thm: generalloss}.

\begin{restatable}[Lower Bound for General Loss]{thm}{generalloss}
\label{thm: generalloss}
Assume that there exists a function $f$ such that $f(\x) \ci \z$ holds, and $f(\x) \ci \z$ holds if and only if $\E \gOne(\rhoOne(\z, W^*_{\z, f(\x)} \cFuntion(\x)))$ reaches the maximum.
For the function $\gTwo $ and $\gOne$, assume that $\gTwo \gOne^{-1}$ is convex and non-decreasing.
Besides, we assume that 
there exists a linear transformation $\lin \in \R^{\dimz \times \dimz}$ such that:

\begin{equation*}
\medmath{
    \begin{split}
        \E [\rhoOne (\bar{\y}, \lin \z)] 
        <& \frac{1}{M} \gTwo \gOne^{-1} \max_f \E \left[\gOne \rhoOne \left(W^*_{\lin \z, f(\x)} f(\x), \lin \z \right)\right] \\
        &- \frac{1}{M}\min_f \E \left[\gTwo\rhoTwo \left(W^*_{\y,f(\x)} \cFuntion(\x), \y \right)\right],
    \end{split}}
\end{equation*}
where $\bar{\y} = (\y, \vec{0}_{\dimz-\dimy}) $ denotes the augmented vector of $\y$ with zero padding, 
and $M$ denotes the upper bound of the derivative of $g_2$, namely, $M \geq {\rm{d}} \gTwo (x)/ \mathrm{d}x, \forall x \in [0, \medmath{\sup_{\x, \y} \rhoTwo (W^*_{\y, \hatcFuntion(\x)} \hatcFuntion(\x), \y) + \sup_{\y, \z} \rhoOne (\bar{\y}, \lin \z)]}$.

Let $ \mathcal{A}^{g}_{\mathcal{P}}$ denote the set of processor that minimize the general training loss with $n_0 = o(\dimc)$ downstream samples and infinite pretext samples, where $\dimc$ is the dimension of $f(\x)$,
\begin{equation*}
    \mathcal{A}^{g}_{\mathcal{P}} = \left\{\cFuntion: \cFuntion \in \argmin_\cFuntion \overline{\PopulationLoss}_{\infty,\ext}(\cFuntion; {\mathcal{P}})  \right\}.
\end{equation*}
Let $\mathcal{B}_{\mathcal{P}}$ be defined as in Theorem~\ref{thm: rationality}.
Then there \emph{exists} a distribution $\mathcal{P}^0$ such that all the function that minimizes the loss $\overline{\PopulationLoss}_{\infty,\ext}(\cFuntion; {\mathcal{P}})$ cannot meet Criterion~\eqref{C1} and \eqref{C2} simultaneously, namely:
$$ \mathcal{A}^{g}_{\mathcal{P}^0} \cap  \mathcal{B}_{\mathcal{P}^0} = \phi.$$

\end{restatable}

Although we use linear transformation $V$ in Theorem~\ref{thm: generalloss}, it can be replaced with any other forms as long as it defines a proper distance between $\z$ and $\y$.

\begin{remark}
The main purpose of Theorem~\ref{thm: generalloss} is to show that the failure in Theorem~\ref{thm: warm-up} does not come from the loss form. Instead, the failure indeed comes from insufficient downstream data used in processor training since various loss forms all lead to failure.
Therefore, we do not provide the rationality of Theorem~\ref{thm: generalloss}.
\end{remark}

\subsection{Model-Dependent Results}\label{sec:model-based result}

The results stated in Section~\ref{sec:warm-up} hold without any assumption of model structures (i.e., the function class of $f$).
This section considers the model structures and provides fine-grained analysis.
Before diving into the main theorem, we introduce a measure of \emph{model capacity} at first, which is defined as follows.

\begin{definition}[Model Capacity]
\label{Def: ModelCapacity}
Define the model capacity $\mathcal{M} (\mathcal{F},\PopulationLoss)$ of function class $\mathcal{F}$ with respect to loss function $\PopulationLoss$ as
\begin{equation*}
\medmath{
    \mathcal{M} (\mathcal{F},\PopulationLoss) 
= \sup \left\{ n:\forall \mathcal{D}, \inf_{\cFuntion \in \mathcal{F}} \sup_{(X,Y)\in \mathcal{D}^n} \mathcal{L}(\cFuntion(X), Y) = 0 \right\},}
\end{equation*}
where $\mathcal{D}$ is the data distribution, and $X,Y$ are data vectors, each of which consists of $n$ samples.
\end{definition}

Intuitively, model capacity measures how well the function class fits the noise under finite samples.
It is similar to the VC dimension under regression settings.
In the following Lemma~\ref{lem:caseCapacity}, we provide the lower bounds of the model capacity for linear models and neural networks.

\begin{restatable}[Model Capacity]{lem}{NNcapacity}
\label{lem:caseCapacity}
Let $\mathcal{J} = \{g: g(\x) = w^\top \x, \x \in \mathbb{R}^d, w \in \mathbb{R}^{k \times d} \}$ be the class of linear functions, its model capacity satisfies
\begin{equation*}
    \mathcal{M} (\mathcal{J},\PopulationLoss) \geq d.
\end{equation*}
Let $\mathcal{K}$ be the class of two-layer neural networks with $k$ neurons, its model capacity satisfies
\begin{equation*}
    \mathcal{M} (\mathcal{K},\PopulationLoss) \geq k/4.
\end{equation*}
Let $\mathcal{K}_m$ be the class of multi-layer neural networks with $L$ layers and $k_i$ neurons in each layer $i$, its model capacity satisfies
\begin{equation*}
    \mathcal{M} (\mathcal{K}_m,\PopulationLoss) \geq \min_{i \in [L]} k_i/4.
\end{equation*}
\end{restatable}

We next provide the model-dependent lower bound in Theorem~\ref{thm: model}:
\begin{restatable}[Model-Dependent Lower Bound]{thm}{model}
\label{thm: model}
Assume that the function class of the processor can be decomposed as $\cF = \cF_1 \times \cF_2$, where there exists a function $f_{0} \in \cF$ such that $\Cov[f_{0}(\x), \y] \not = 0$ and a function $f_2 \in \cF_2$ such that $f_2(\x) \ci \z$.
Assume matrix $\E[\cFuntion(\x)\cFuntion^\top(\x)]$ is nonsingular for all $\cFuntion \in \cF$.
Let $\mathcal{A}^\prime_{\mathcal{P}}$ be the set of processor that minimize the loss with $n_0 = o(\mathcal{M}(\cF_1, \mathcal{L}))$ downstream samples and infinite pretext samples, where $\dimc$ is the dimension of $f$,
\begin{equation*}
    \mathcal{A}^{\prime}_{\mathcal{P}} = \left\{\cFuntion: \cFuntion \in \argmin_\cFuntion \PopulationLoss_{\infty,\ext}(\cFuntion; {\mathcal{P}})  \right\}.
\end{equation*}
And let $\mathcal{B}_{\mathcal{P}}, \mathbb{S}$ be defined as in Theorem~\ref{thm: rationality}.
Then there \emph{exists} a distribution $\mathcal{P}^0 \in \mathbb{S}$ (which means $ \mathcal{A}_{\mathcal{P}^0 }\subset  \mathcal{B}_{\mathcal{P}^0} $), such that all the function that minimizes the loss $\PopulationLoss_{\infty,\ext}$ cannot meet Criterion~\eqref{C1} and \eqref{C2} simultaneously, namely:
$$ \mathcal{A}^\prime_{\mathcal{P}^0} \cap  \mathcal{B}_{\mathcal{P}^0} = \phi.$$
\end{restatable}

The function class $\mathcal{F}$ is decomposed into two parts:
The inner one $\mathcal{F}_2$ is the minimal class of functions in which there exists a function that can remove all the information of $\z$, and the outer one $\mathcal{F}_1$'s model capacity decides the lower bound of the downstream sample size.
Usually, the information of $\x$ is much more than the information of $\z$, and therefore, $\mathcal{F}_2$ is expected to be not complex.

Intuitively, Theorem~\ref{thm: model} shows that with a more complex function class to train the processor, the lower bound for downstream samples is higher.
Note that the model capacity is infinity for the neural network with infinite width.
Therefore, the procedure provably fails due to the large model capacity.

\begin{remark}
The model-dependent bound (Theorem~\ref{thm: model}) and the model-free bound (Theorem~\ref{thm: warm-up}) focus on the processor training and the pretext step, respectively.
Therefore, one cannot reach the conclusion that Theorem~\ref{thm: model} dominate Theorem~\ref{thm: warm-up}.
However, since we usually use neural networks to train processor $f$, which usually suffers from a large model capacity, Theorem~\ref{thm: model} is better than Theorem~\ref{thm: warm-up} with large model capacity.
\end{remark}

%% file: text/5experiments.tex
\newcommand{\tabincell}[2]{\begin{tabular}{@{}#1@{}}#2\end{tabular}}

\section{EXPERIMENTS}\label{sec:experiments}
In this section, we conduct experiments on synthetic data and real-world data, mainly to show that providing few downstream samples hurts the model performance.
In each experiment, we run each experiment 5 times repeatedly and calculate its mean and 95\% confidence bands. 
We defer the experiment details to the supplementary materials.

\subsection{Synthetic Data}
We validate the statements proved before in this section, showing that the processor-based learning fails when (a) training processor with insufficient downstream samples or (b) the penalty parameter is large.

\textbf{Setup.} We consider a linear regime, where $\x$ is generated from Gaussian distribution, $\z = \x[0:\dimz-1] + {\epsilon}_{\z}$ and $\y = \x[0:\dimy-1] + \epsilon_{\y}$, where $\epsilon_{\z}, \epsilon_{\y}$ are generated from Gaussian distributions.
Figure~\ref{fig: experiment} summarizes the experiments (blue and yellow) with their 95\% confidence bands (light blue and light yellow).
We refer to supplementary materials for more details.

\textbf{Algorithm and Metrics.}
We run our newly proposed algorithm (labeled as Ours), which first trains the processor and then does pretext and downstream tasks.
We denote $\dimc$ as the coarse representation dimension and $\lambda$ as the penalty coefficient in the loss (See Equation~\eqref{eqn: lossterm}).
We also run the standard self-supervised learning (labeled as SSL).
In terms of metrics, we calculate the MSE on test downstream samples as the model performance.
Intuitively, MSE is small when the features are learned well.

\textbf{Analysis.}
We plot the results in Figure~\ref{fig: experiment} and defer the specific statistics in the supplementary materials due to space limitations.
Firstly, we plot the MSE as the number of samples used in processor training $n_0$ varies from 30 to 140 in Figure~\ref{fig: changenxy}.
With large $n_2$, the downstream tasks benefit from small MSE, showing that the algorithm indeed finds the proper coarse representations (demonstrated in Theorem~\ref{thm: rationality}).
Secondly, we test the case $\lambda$ varies from $0.1$ to $1.5$, as plotted in Figure~\ref{fig: changelamb}.
When $\lambda$ is large, the algorithm suffers from unsatisfying performances as demonstrated in Theorem~\ref{thm: rationality}.
Thirdly, we test the case $\dimc$ varies from $1$ to $12$, and plot them in Figure~\ref{fig: changec}.

\begin{figure*} 
  \centering 
  \subfigure[MSE under different $\ext$]{ 
\label{fig: changenxy}
    \includegraphics[width=0.31\textwidth]{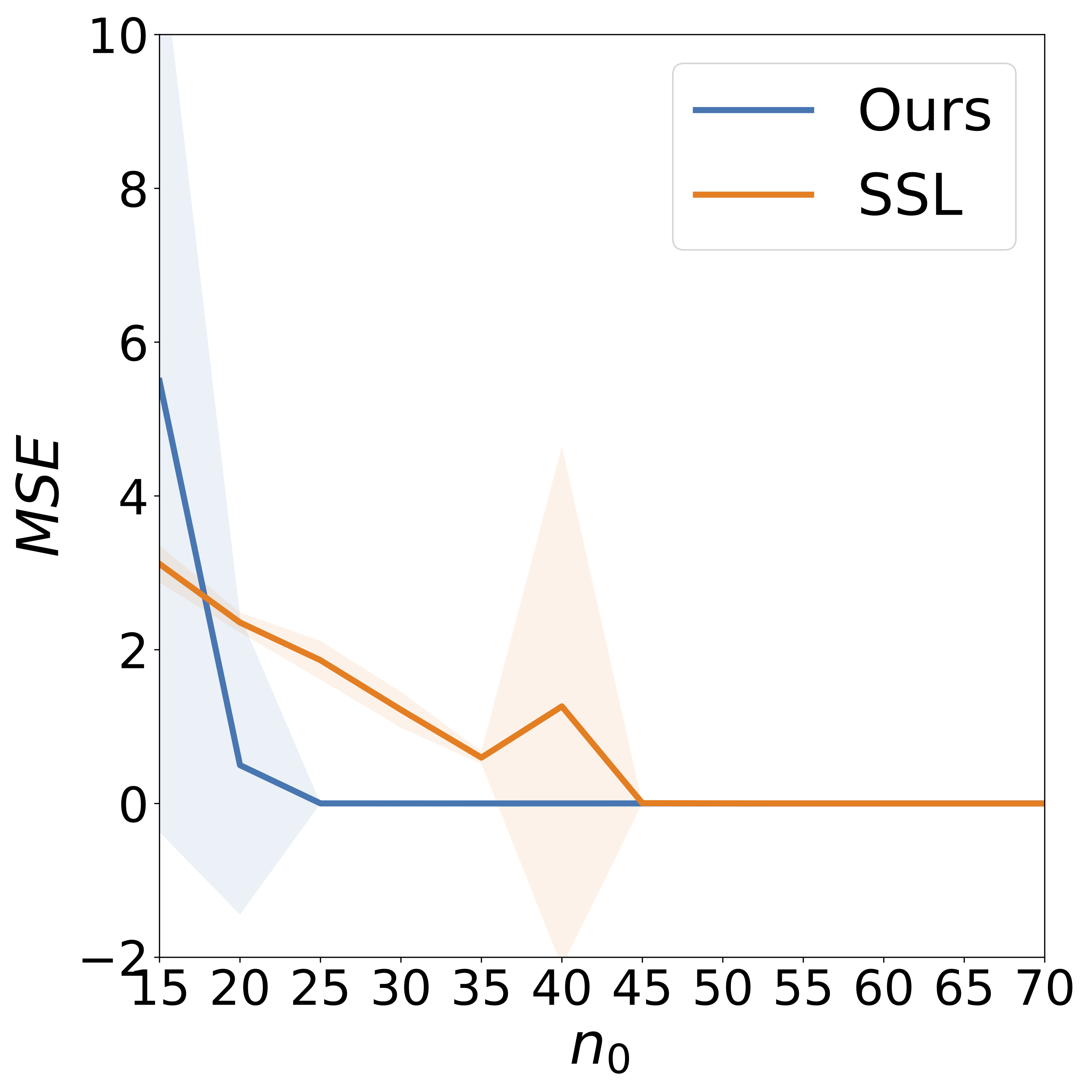}} 
  \subfigure[MSE under different $\lambda$]{ 
\label{fig: changelamb}
    \includegraphics[width=0.31\textwidth]{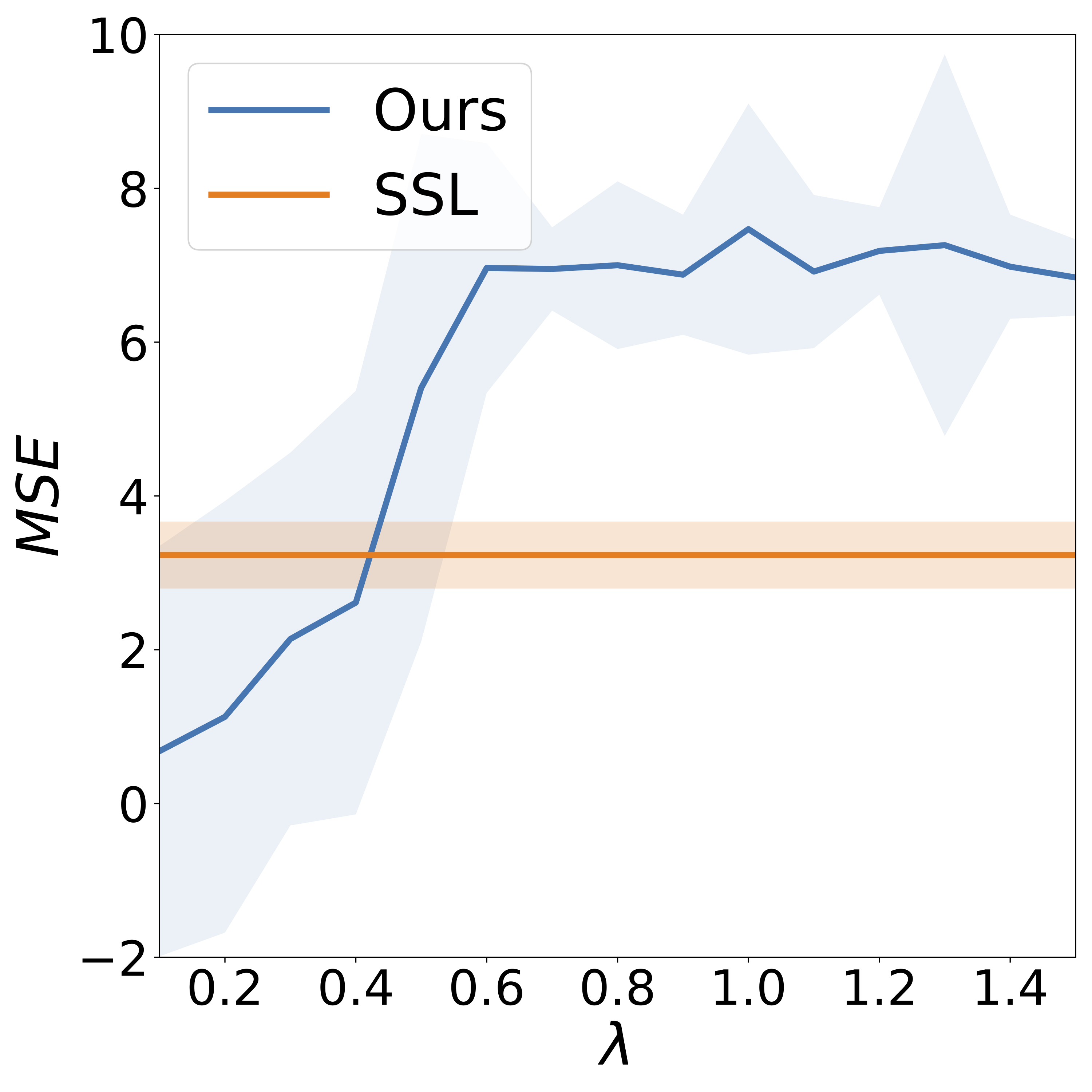}} 
    \subfigure[MSE under different $\dimc$]{ 
\label{fig: changec} 
    \includegraphics[width=0.31\textwidth]{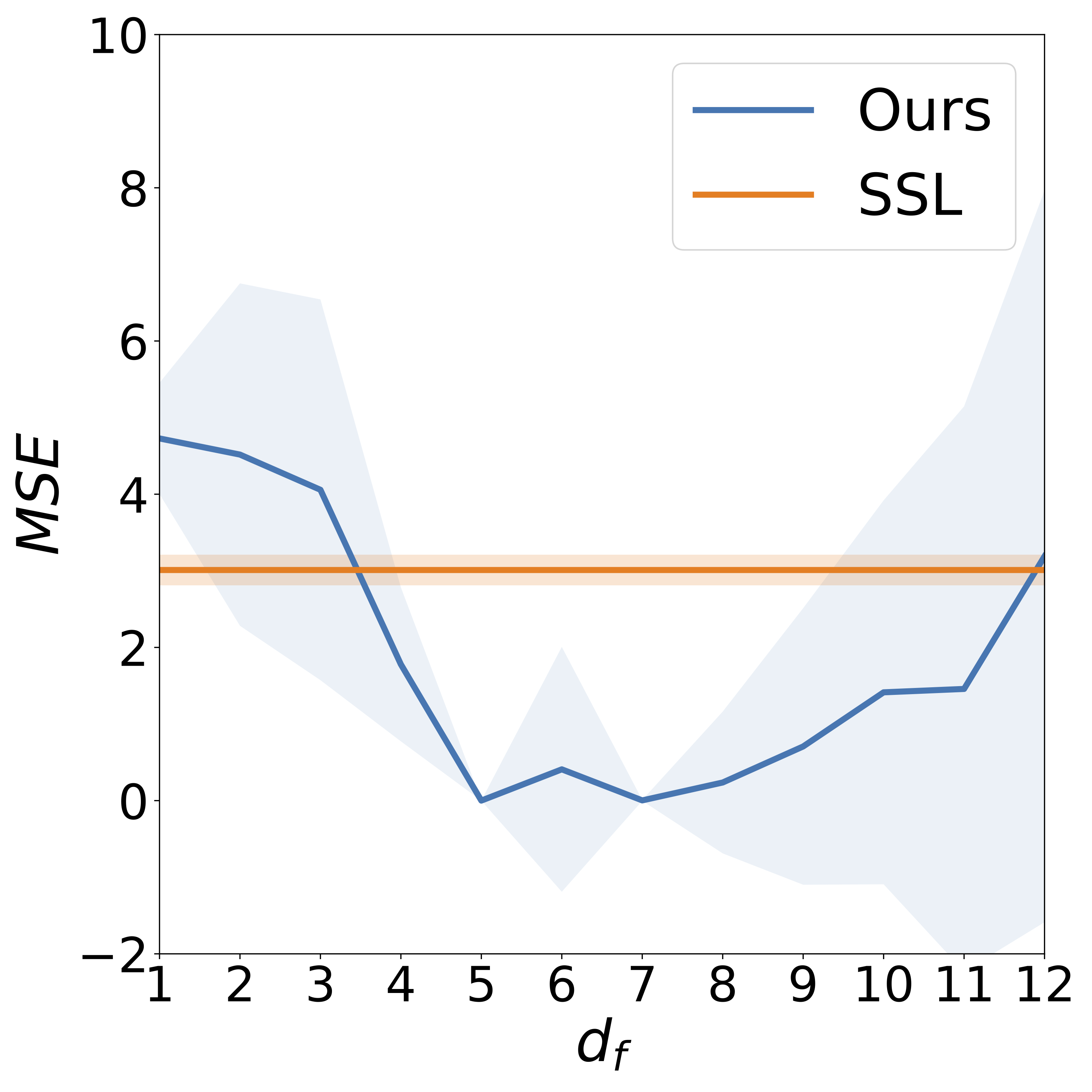}}
  \caption{\textbf{Algorithm performance under different hyperparameters on synthetic data.} (a) Our algorithm learns better with sufficient downstream samples (large $\ext$). MSE decreases to zero as downstream samples increase. (b) A large penalty forces the coarse representation to abandon the $\y$ information, leading to a large MSE. (c) When $\dimc$ is small, the model underfits; when $\dimc$ is large, the model suffers from a limited number of downstream samples.} 
  \label{fig: experiment} 
\end{figure*}

\subsection{Real-World Data (CIFAR-10)}
In this section, we conduct experiments on CIFAR-10~\citep{cifar} and show that the newly proposed processor-based method indeed fails in real-world tasks.
We demonstrate that the performance gets worse with (a) larger penalty parameter $\lambda$, (b) larger model capacity, and (c) training processors with insufficient downstream samples.

\textbf{Setup.}
We conduct experiments on CIFAR-10 and choose rotation prediction~\citep{rotation} as the baseline SSL framework.
As previously described, the training set is split into pretext samples (unlabeled) and downstream samples (labeled) without overlap to mimic the SSL training and a downstream classification task. 

The baseline follows~\citet{rotation}, where we train a plain rotation prediction task using 30k unlabeled training samples on a four-block NIN model~\citep{NINmodel}. Then we treat the 15k labeled training samples and 10k test samples as the training and test sets of the downstream classification task. A linear classifier is learned to conduct the task\footnote{To fully take advantage of the downstream samples, we require that the sum of downstream samples used in processor training and the downstream samples used in the downstream task are 20k.}. 
For the processor-based learning, the only difference is that we now specify $f$ in Equation~\eqref{eqn: lossterm} as the first two blocks in a NIN model. 
When optimizing $f$, the 5k labeled and 30k unlabeled training samples are used to minimize Loss in Equation~\eqref{eqn: lossterm}. After obtaining $f$, we fix the weights of the first two blocks in the NIN model and go through the SSL training pipeline again, same as above.

\begin{table}[t]
    \caption{\textbf{Large lambda, large model capacity hurt the performance} with 15k downstream data used in processor training. The newly proposed processor-training method gets worse with larger model, while SSL becomes better.}
    \label{tab:lambdaCIFAR}
    \centering
    \begin{tabular}{c|c|c|c|c}
    \hline
    $\lambda$& 0.001&1&10& SSL \\
    \hline
      Full  & \tabincell{c}{44.48\\(0.84)} &	\tabincell{c}{23.85\\(6.17)} &	\tabincell{c}{22.29\\(6.21)}  & \textbf{\tabincell{c}{74.28\\(0.06)}} \\
    \hline
      Double  & \tabincell{c}{38.72\\(0.78)} & \tabincell{c}{26.89\\(5.44)} & \tabincell{c}{16.86\\(5.78)} &  \textbf{\tabincell{c}{77.88\\(0.10)}}\\
      \hline
    \end{tabular}
\end{table}

\begin{table}[t]
\caption{
    \textbf{
    Training processors with insufficient downstream data hurts the performance.} With more samples used in processor training ($n_0$), the performance becomes better. However, they do not exceed standard SSL in CIFAR-10.}
    \label{tab:n0CIFAR}
    \centering
    \begin{tabular}{c|c|c|c|c|c}
    \hline
   $\ext$ & 1k& 5k& 10k& 15k & SSL \\
    \hline
     acc & \tabincell{c}{42.49\\ (1.30)}  & \tabincell{c}{43.38\\ (0.80)} & \tabincell{c}{44.76\\ (0.65)} & \tabincell{c}{44.48\\ (0.84)} & \textbf{\tabincell{c}{74.28\\ (0.06)}} \\
      \hline
    \end{tabular}
    
\end{table}

\textbf{Algorithm and Metrics.}
We use the test accuracy (acc) to evaluate the model performance.
We run the newly proposed processor training techniques with different $\lambda$ (see Equation~\eqref{eqn: lossterm}) and different model capacity (see Definition~\ref{Def: ModelCapacity}) in Table~\ref{tab:lambdaCIFAR}.
We label the original model as \emph{Full} and label the double-size model as \emph{Double}.
Obviously, Double has a larger model capacity compared to Full.
Besides, we run the newly proposed method under different downstream samples used in processor training in Table~\ref{tab:n0CIFAR}.

\textbf{Analysis.}
We list the results in Table~\ref{tab:lambdaCIFAR} and Table~\ref{tab:n0CIFAR} and defer the specific statistics in the supplementary materials.
Firstly, we test the role of $\lambda$ in Table~\ref{tab:lambdaCIFAR}, demonstrating that large $\lambda$ indeed hurts the model performance (as shown in Theorem~\ref{thm: rationality}).
Secondly, we validate Theorem~\ref{thm: warm-up} in Table~\ref{tab:n0CIFAR}, showing that a training processor with fewer downstream samples $\Xss$ results in worse performance.
We finally test the statement in Theorem~\ref{thm: model} by Table~\ref{tab:lambdaCIFAR}, showing that when double the model size, the performance gets much worse.

%% file: text/2relatedwork.tex
\section{RELATED WORKS}

\textbf{Self-Supervised Methods in Practice.}
There are three common approaches for Self-Supervised Learning (SSL):
generative model based, contrastive learning based, and pretext based.
Generative model based SSL \citep{donahue2016adversarial,dumoulin2016adversarially,donahue2019large} learns a bijective mapping between input and representation.
Contrastive learning based SSL learns representations by maximizing the mutual information between the global and local features \citep{hjelm2018learning,oord2018representation,bachman2019learning} or between the features of positive samples \citep{tian2019contrastive,he2020momentum,chen2020simple}.
Pretext based SSL learns representations via handcrafted pretext tasks \citep{zhang2016colorful,pathak2016context,noroozi2016unsupervised,gidaris2018unsupervised}.
These three approaches
are quite different in technique.
For example, contrastive learning based approaches are trained by comparing positive (and negative) pairs, while
pretext-based approaches assign a label for each unlabeled sample.
Such difference leads to each sample with a pretext label loss in pretext-based SSL, while in contrastive learning-based SSL, a batch of unlabeled samples produces a loss. 
In this paper, we focus on studying the pretext-based SSL.

\textbf{Theory for Self-Supervised Learning.}
Although there are a number of great empirical works for SSL,
the theoretical study of SSL is still at an early stage.
The most related work to ours is given by
\citet{lee2020predicting}.
They are the first to formulate pretext-based SSL, and show that it can reduce the sample complexity of downstream tasks compared with supervised learning.
They also point out that when the CI condition holds, the downstream sample complexity achieves the optimal,
and it gets worse when the CI condition does not hold.
In this paper, we further study the situation that the CI condition does not hold and
explore the idea of applying a learnable function to the input to make the CI condition hold. 
Furthermore, \citet{saunshi2020mathematical} study the specific pretext task of next word prediction and the downstream task of text classification.
Other works \citep{arora2019theoretical,tian2021understanding,wang2022chaos,huang2021towards} study the generalization error of contrastive learning based SSL, whose setting is different from our paper.
Moreover, \citet{bansal2020self} analyze the generalization gap for most SSL methods. However, the rationality gap, which is a part of the generalization gap, cannot be theoretically bounded.

%% file: text/6conclusion.tex
\section{FUTURE WORK}

In this work, we explore the idea of using part of downstream data to boost the pretext-based self-supervised learning by making the conditional independence $ f(\x)\ci \z \midbar \y$ hold.
We show that taking limited downstream data provably hurts the performance and give both model-free and model-dependent lower bounds of sample size.
One possible future work is to rigorously prove that the pretext based self-supervised learning can be boosted
with sufficient downstream data,
suggested by the experiments (Figure~\ref{fig: experiment}).
Moreover, one can consider generalizing this paper to the contrastive learning regime to see if involving part of downstream information in the pre-training can boost the downstream performance.

%% file: supplement.tex
We give all the proofs of lemmas and theorems here organized by theorems, i.e., Section~\ref{Proof: thmrationality} for
Theorem~\ref{thm: rationality} and related lemmas, Section~\ref{proof: warm-up} for Theorem~\ref{thm: warm-up}, Section~\ref{proof: general} for Theorem~\ref{thm: generalloss}, and Section~\ref{proof: model} for Theorem~\ref{thm: model}.

\input{text/Appendix/pfthm1}

\input{text/Appendix/pfthm2}
\input{text/Appendix/pfthm3}

\input{text/Appendix/pfthm4}
\input{text/Appendix/example}
\input{text/Appendix/AppendixExperiment1}

\input{text/Appendix/appendix}

%% file: text/Appendix/pfthm1.tex
\section{Proof of Theorem~\ref{thm: rationality}}
\label{Proof: thmrationality}

\Rationality*

To prove Theorem~\ref{thm: rationality}, we need to construct a data distribution $\mathcal{P}_0$ such that $\mathcal{P}_0 \in \mathbb{S}$. We construct such distribution as follows: 
Let $(\x, \y, \z)$ be the input variables, pretext label, and downstream label, respectively.
Firstly, let $\E[\x]=\E[\y]=\E[\z]=0$ and $\E [\z\z^\top]\! =\! I$.
Given a marginal distribution of $(\x, \y)$, construct a linear relationship $\y = \bM\z$ where the singular values of matrix $\bM$ be $\sigma_1 =  \dots= \sigma_\dimy = \sigma$.
And choose the penalty $\lambda < \sigma^2$.

We first take a closer look at the loss function $\PopulationLoss(f)=\PopulationLoss_2+\lambda\PopulationLoss_1$. 
We can rewrite it as 
$\PopulationLoss(f)=(1-\frac{\lambda}{\sigma^2})\PopulationLoss_2+\lambda (\PopulationLoss_1+\frac{1}{\sigma^2}\PopulationLoss_2)$. 
The first term is $\PopulationLoss_2$ multiplied by a coefficient, which still captures the information of $\y$ as $\PopulationLoss_2$ does.
The second term intuitively captures the redundant information of $\z$.

We first provide Lemma~\ref{lem: lem1} and Lemma~\ref{lem: lem2} used during the proof of Theorem~\ref{thm: rationality}.
Lemma~\ref{lem: lem1} shows that $\PopulationLoss_1+\frac{1}{\sigma^2}\PopulationLoss_2$ is minimized if and only if all the redundant information of $\z$ is eliminated.

\begin{restatable}{lem}{Lema}
\label{lem: lem1}
Under the assumptions of Theorem~\ref{thm: rationality} and data distribution $\mathcal{P}_0$ with penalty $\lambda < \sigma^2$, then $\PopulationLoss_1+\frac{1}{\sigma^2}\PopulationLoss_2$ is minimized if and only if the conditional independence criterion~\eqref{C1} holds,
i.e., $\Cov[\cFuntion(\x), \z \midbar \y]=0$.
\end{restatable}

The advantage of decomposing $\PopulationLoss(f)$ into $(1-\frac{\lambda}{\sigma^2})\PopulationLoss_2$ and  $\PopulationLoss_1+\frac{1}{\sigma^2}\PopulationLoss_2$ is that 
these two terms can be optimized individually.
In other words, when $\PopulationLoss(f)$ is minimized, the above two terms are also minimized.
We state it formally as the following Lemma~\ref{lem: lem2}.

\begin{restatable}{lem}{Lemb}
\label{lem: lem2}
Under the assumptions of Theorem~\ref{thm: rationality} and data distribution $\mathcal{P}_0$ with penalty $\lambda < \sigma^2$, 
if $\PopulationLoss(f)$ is minimized,
then  $(1-\frac{\lambda}{\sigma^2})\PopulationLoss_2$ and  $\PopulationLoss_1+\frac{1}{\sigma^2}\PopulationLoss_2$ are both minimized.
\end{restatable}

According to Lemma~\ref{lem: lem2}, for each function $f$ that minimizes the loss, it also minimizes $(1-\frac{\lambda}{\sigma^2})\PopulationLoss_2$ and  $\PopulationLoss_1+\frac{1}{\sigma^2}\PopulationLoss_2$ at the same time.
For the first term, since $1-\frac{\lambda}{\sigma^2}$ is a positive coefficient, $\PopulationLoss_2$ is minimized.
Therefore, Criterion~\eqref{C2} holds.
Since the second term is minimized, Criterion~\eqref{C1} holds by Lemma~\ref{lem: lem1}.
Thus, such $f$ satisfies both Criterion~\eqref{C1} and \eqref{C2}.
We next  prove Lemma~\ref{lem: lem1} and Lemma~\ref{lem: lem2} in Section~\ref{Proof: lem1} and Section~\ref{Proof: lem2}, respectively,

\subsection{Proof of Lemma~\ref{lem: lem1}}
\label{Proof: lem1}
\Lema*

\begin{proof}
Criterion~\eqref{C1} is equivalent to
\begin{equation*}
\Sigma_{\cFuntion\left(\x\right), \z | \y} = \Sigma_{\z \cFuntion\left(\x\right)} - \Sigma_{\y \cFuntion\left(\x\right)}\Sigma^{-1}_{\y \y}\Sigma_{\y \z} = 0.
\end{equation*}
Notice that $\E[\y]=\E[\z]=\E[f(\x)]=0$, thus the above equation can be rewritten as
\begin{equation}
\label{Eqn:independence}
\E[\cFuntion\left(\x\right)\z^\top] = \E[\cFuntion\left(\x\right)\y^\top]\left(\E [\y\y^\top]\right)^{-1}\E[\y\z^\top].
\end{equation}

On the other hand, we express the term $\mathcal{L}_1+\frac{1}{\sigma^2} \mathcal{L}_2$ as
\begin{equation*}
\begin{split}
     &\mathcal{L}_1+\frac{1}{\sigma^2} \mathcal{L}_2 \\
    =&~   \frac{1}{\sigma^2} \E\norm{ \y - W^*_{\y, \cFuntion\left(\x\right)} \cFuntion\left(\x\right)}^2 - \E\norm{ \z - W^*_{\z, \cFuntion\left(\x\right)} \cFuntion\left(\x\right)}^2 \\
    =&~  \tr[\frac{1}{\sigma^2}\E\left[\left(\y - W^*_{\y, \cFuntion\left(\x\right)} \cFuntion\left(\x\right)\right)\left(\y - W^*_{\y, \cFuntion\left(\x\right)} \cFuntion\left(\x\right)\right)^\top\right] \\
    &-\E\left[\left(\z - W^*_{\z, \cFuntion\left(\x\right)} \cFuntion\left(\x\right)\right)\left(\z - W^*_{\z, \cFuntion\left(\x\right)} \cFuntion\left(\x\right)\right)^\top\right]]\\
    =&~  \tr[ \frac{1}{\sigma^2} \left(\E [\y\y^{\top}] - \E\left[\y \cFuntion^\top\left(\x\right)\right] \left(\E[\cFuntion \left(\x\right)\cFuntion^\top \left(\x\right)]\right)^{-1} \E\left[\cFuntion\left(\x\right)\y^\top\right]\right) \\
    &- \left(\E [\z\z^{\top}]  -\E\left[\z \cFuntion^\top\left(\x\right)\right] \left(\E[\cFuntion \left(\x\right)\cFuntion^\top \left(\x\right)]\right)^{-1} \E\left[\cFuntion \left(\x\right)\z^\top\right]\right) ]\\
    =&~  \tr [\E\left[\z \cFuntion^\top\left(\x\right)\right] \left(\E[\cFuntion \left(\x\right)\cFuntion^\top \left(\x\right)]\right)^{-1} \E\left[\cFuntion \left(\x\right)\z^\top\right] \\
    &- \frac{1}{\sigma^2}\E\left[\y \cFuntion^\top\left(\x\right)\right] \left(\E[\cFuntion \left(\x\right)\cFuntion^\top \left(\x\right)]\right)^{-1} \E\left[\cFuntion\left(\x\right)\y^\top\right]] +  \left(\frac{1}{\sigma^2}\E [\y^{\top}\y] - \E[ \z^{\top}\z] \right)\\
    =&~ \medmath{ \tr\left[ \left(I - \frac{1}{\sigma^2}\bM^\top \bM\right) \E\left[\z \cFuntion^\top\left(\x\right)\right] \left(\E[\cFuntion \left(\x\right)\cFuntion^\top \left(\x\right)]\right)^{-1} \E\left[\cFuntion \left(\x\right)\z^\top\right]\right] +  \left(\frac{1}{\sigma^2}\E [\y^{\top}\y] - \E [\z^{\top}\z] \right)},
\end{split}
\end{equation*}
where the third equation holds because 
$W^*_{\y, \cFuntion\left(\x\right)}=\E[\y f^\top\left(\x\right)]\left(\E[f\left(\x\right)f^\top\left(\x\right)]\right)^{-1}$
and the last equation follows the assumption that $\y = \bM \z$.
Notice that the second term $\frac{1}{\sigma^2} \medmath{\E \left[\y^{\top}\y\right] - \E \left[\z^{\top}\z\right]} $ is unrelated to $\cFuntion$.
Therefore, minimizing $\mathcal{L}_1+\frac{1}{\sigma^2} \mathcal{L}_2$ is equivalent to minimizing 
\begin{equation}
\label{EQN: T}
T := \tr\left[ \left(I - \frac{1}{\sigma^2}\bM^\top \bM\right) \E\left[\z \cFuntion^\top\left(\x\right)\right] \left(\E\left[\cFuntion \left(\x\right)\cFuntion^\top \left(\x\right)\right]\right)^{-1} \E\left[\cFuntion \left(\x\right)\z^\top\right]\right].
\end{equation}

(a) We first prove the \textbf{sufficient condition}, namely, the conditional independence criterion \eqref{C1} leads to minimizing $\mathcal{L}_1+\frac{1}{\sigma^2} \mathcal{L}_2$.

Plugging Equation~\eqref{Eqn:independence} to Equation~\eqref{EQN: T}, we have
\begin{equation*}
\small
    \begin{split}
T &= \tr \left[\left(I - \frac{1}{\sigma^2}\bM^\top \bM\right) \E\left[\z \cFuntion^\top\left(\x\right)\right] \left(\E\left[\cFuntion \left(\x\right)\cFuntion^\top \left(\x\right)\right]\right)^{-1} \E\left[\cFuntion \left(\x\right)\z^\top\right]\right] \\
&= \tr\left[ \E\left[\cFuntion \left(\x\right)\z^\top\right] \left(I - \frac{1}{\sigma^2}\bM^\top \bM\right) \E\left[\z \cFuntion^\top\left(\x\right)\right] \left(\E\left[\cFuntion \left(\x\right)\cFuntion^\top \left(\x\right)\right]\right)^{-1} \right]\\
&= \medmath{\tr\left[ \E\left[\cFuntion\left(\x\right)\y^\top\right]\left(\E [\y\y^\top]\right)^{-1}\E[\y\z^\top] \left(I - \frac{1}{\sigma^2}\bM^\top \bM\right) 
\left(\E\left[\cFuntion\left(\x\right)\y^\top\right]\left(\E [\y\y^\top]\right)^{-1}\E[\y\z^\top]\right)^\top \left(\E\left[\cFuntion \left(\x\right)\cFuntion^\top \left(\x\right)\right]\right)^{-1} \right]}.
    \end{split}
\end{equation*}

By assumption that $\y = \bM \z$ and $\E [\z\z^\top] = I$, we have
\begin{equation*}
\begin{split}
&~\E\left[\cFuntion\left(\x\right)\y^\top\right]\left(\E [\y\y^\top]\right)^{-1}\E[\y\z^\top] \left(I - \frac{1}{\sigma^2}\bM^\top \bM\right) 
\left(\E\left[\cFuntion\left(\x\right)\y^\top\right]\left(\E [\y\y^\top]\right)^{-1}\E[\y\z^\top]\right)^\top\\
    =&~ \medmath{\E\left[\cFuntion\left(\x\right)\z^\top\right]\bM^\top \left(\E [\bM\z\z^\top\bM^\top]\right)^{-1} \bM \E[\z\z^\top] \left(I - \frac{1}{\sigma^2}\bM^\top \bM\right) \left(\E\left[\cFuntion\left(\x\right)\z^\top\right]\bM^\top \left(\E [\bM\z\z^\top\bM^\top]\right)^{-1}\bM \E[\z\z^\top]\right)^\top} \\
=&~ \E\left[\cFuntion\left(\x\right)\z^\top\right]\bM^\top [\bM\bM^\top]^{-1} \bM  \left(I - \frac{1}{\sigma^2}\bM^\top \bM\right) \left(\E\left[\cFuntion\left(\x\right)\z^\top\right]\bM^\top [\bM\bM^\top]^{-1}\bM\right)^\top \\
=&~ \E\left[\cFuntion\left(\x\right)\z^\top\right]\bM^\top [\bM\bM^\top]^{-1}  \left[\bM \bM^\top  - \frac{1}{\sigma^2}\bM \bM^\top \bM \bM^\top \right] [\bM\bM^\top]^{-1} \bM \E\left[\z \cFuntion\left(\x\right)^\top\right] \\
=&~ \E\left[\cFuntion\left(\x\right)\z^\top\right]\bM^\top [\bM\bM^\top]^{-1} \cdot 0_{\dimy\times\dimy} \cdot [\bM\bM^\top]^{-1} \bM \E\left[\z \cFuntion\left(\x\right)^\top\right]\\
=&~ 0.
\end{split}
\end{equation*}

Furthermore, notice that the eigenvalues of $I - \frac{1}{\sigma^2}\bM^\top \bM$ is no less than zero by the definition of $\sigma$. And notice that the eigenvalues of $\E\left[\z \cFuntion^\top\left(\x\right)\right] \left(\E\left[\cFuntion \left(\x\right)\cFuntion^\top \left(\x\right)\right]\right)^{-1} \E\left[\cFuntion \left(\x\right)\z^\top\right]$ is also no less than zero.
Therefore, $T$ reaches the minimum when it reaches zero.
To conclude, the sufficient condition holds.

(b) We next prove the \textbf{necessary condition}, namely, 
minimizing $\mathcal{L}_1+\frac{1}{\sigma^2} \mathcal{L}_2$ leads to the conditional independence.

Under the assumption that there exists a ground truth $\cFuntionStar$ satisfying Criterion~\eqref{C1} and \eqref{C2}, we see from the sufficient condition that when $T$ reaches the minimum, $T$ must be equal to zero:
\begin{equation*}
 T=\tr\left[ \left(I - \frac{1}{\sigma^2}\bM^\top \bM\right) \E\left[\z \cFuntion^\top\left(\x\right)\right] \left(\E\left[\cFuntion \left(\x\right)\cFuntion^\top \left(\x\right)\right]\right)^{-1} \E\left[\cFuntion \left(\x\right)\z^\top\right]\right] = 0.   
\end{equation*}

Since the matrix is semi-definite, we can omit the trace term and rewrite it as:

\begin{equation*}
\E\left[\cFuntion \left(\x\right)\z^\top\right] \left(I - \frac{1}{\sigma^2}\bM^\top \bM\right) \E\left[\z \cFuntion^\top\left(\x\right)\right] \left(\E\left[\cFuntion \left(\x\right)\cFuntion^\top \left(\x\right)\right]\right)^{-1}  = 0.
\end{equation*}

Besides, since $\E\left[\cFuntion \left(\x\right)\cFuntion^\top \left(\x\right)\right]$ is non-singular, we have
\begin{equation*}
\E\left[\cFuntion \left(\x\right)\z^\top\right] \left(I - \frac{1}{\sigma^2}\bM^\top \bM\right) \E\left[\z \cFuntion^\top\left(\x\right)\right]  = 0.
\end{equation*}

On the other hand, we represent the covariance as the follows based on the assumption $\E\left[\z\z^\top\right] = I$,
\begin{equation*}
\Sigma_{\cFuntion\left(\x\right), \z | \y} =\E\left[\cFuntion\left(\x\right) \z^\top\right] - \E\left[\cFuntion\left(\x\right)\y^\top\right]\left(\E \left[\y\y^\top\right]\right)^{-1}\E\left[\y\z^\top\right] = \E\left[\cFuntion\left(\x\right)\z^\top\right] \left[I -\frac{1}{\sigma^2} \bM^\top \bM\right].
\end{equation*}

Finally, the covariance meets
\begin{equation*}
\Sigma_{\cFuntion\left(\x\right), \z | \y} \Sigma^\top_{\cFuntion\left(\x\right), \z | \y} = \E \left[ \cFuntion\left(\x\right) \z^\top\right] \left[I -\frac{1}{\sigma^2} B^\top B\right]\E \left[ \z \cFuntion\left(\x\right)^\top\right] =0.
\end{equation*}

This leads to the conclusion that $\Sigma_{\cFuntion\left(\x\right), \z | \y} = 0$.

Combining (a) and (b), we finish the proof.
\end{proof}

\subsection{Proof of Lemma~\ref{lem: lem2}}
\label{Proof: lem2}
\Lemb*
\begin{proof}
Notice that there exists a ground truth $\cFuntionStar$ that satisfies both Criterion~\eqref{C1} and Criterion~\eqref{C2}.
By definition of Criterion~\eqref{C2}, the ground truth $\cFuntionStar$ should minimize $\mathcal{L}_2$.
On the other hand, we have proved in Section~\ref{Proof: lem1} that satisfying Criterion~\eqref{C2} leads to minimizing $\mathcal{L}_1 + \frac{1}{\sigma^2} \mathcal{L}_2$.

Notice that $$\mathcal{L}(f) =\lambda\mathcal{L}_1 + \mathcal{L}_2 =\left(1-\frac{\lambda}{\sigma^2}\right) \mathcal{L}_2 + \lambda\left(\mathcal{L}_1 + \frac{1}{\sigma^2} \mathcal{L}_2\right),$$ with $\lambda>0$.
Therefore, the ground truth $\cFuntionStar$ minimizes $\left(1-\frac{\lambda}{\sigma^2}\right) \mathcal{L}_2$ and $\mathcal{L}_1 + \frac{1}{\sigma^2} \mathcal{L}_2$ at the same time.
Thus, any function $f$ minimizing the loss must minimize both $\left(1-\frac{\lambda}{\sigma^2}\right) \mathcal{L}_2$ and $\mathcal{L}_1 + \frac{1}{\sigma^2} \mathcal{L}_2$,
or it would be larger than $\mathcal{L}(\cFuntionStar)$.
\end{proof}

%% file: text/Appendix/pfthm2.tex
\section{Proof of Theorem~\ref{thm: warm-up}}
\label{proof: warm-up}
\warmup*

\begin{proof}

We consider the distribution $\mathcal{P}_0$ as in Section~\ref{Proof: thmrationality}, that is:
Firstly, let $\E[\x]=\E[\y]=\E[\z]=0$ and $\E [\z\z^\top]\! =\! I$.
Given a marginal distribution of $(\x, \y)$, construct a linear relationship $\y = \bM\z$ where the singular values of matrix $\bM$ be $\sigma_1 =  \dots= \sigma_\dimy = \sigma$.
And choose the penalty $\lambda < \sigma^2$.

Let us first consider the training process.
We first claim that when $n_0 = o(\dimc)$, for any $f$, the first term of $\mathcal{L}_{\infty,n_0}$ can be trained to zero, i.e., 
\begin{equation*}
    \frac{1}{n_0}\norm{\yMatrix_{down1} - \widetilde{W}_2 \cFuntion\left(\xMatrix_{down1}\right)}^2 = 0,
\end{equation*}
by taking
$$\widetilde{W}_2 = \yMatrix_{down1}\cFuntion^\top\left(\xMatrix_{down1}\right)\left[\cFuntion\left(\xMatrix_{down1}\right) \cFuntion^\top\left(\xMatrix_{down1}\right)\right]^{-1} ,$$
where we abuse $\cFuntion\left(\xMatrix_{down1}\right) \in \mathbb{R}^{ \dimc\times n_0}$.

Therefore,  to minimize the loss $\mathcal{L}_{\infty, n_2^\prime}$ only needs to minimize the second term of $\mathcal{L}_{\infty,n_0}$:
\begin{equation*}
\hat{f} \in \argmin_f \left[ -\min_W \E_{\x,\z} \left\| \z - W \cFuntion\left(\x\right)  \right\|^2\right].
\end{equation*}

Note that for any $f$,
\begin{equation*}
\min_W \E_{\x,\z} \left\| \z - W \cFuntion\left(\x\right)  \right\|^2 
= \E \norm{\z}^2-\tr \left[\E \left[ \z \cFuntion^\top\left(\x\right)\right] \left(\E \left[ \cFuntion\left(\x\right) \cFuntion^\top\left(\x\right)\right]\right)^{-1} \E \left[ \cFuntion\left(\x\right) \z^\top\right]\right]
\leq \E \| \z \|^2.    
\end{equation*}

When $f\left(\x\right)$ is uncorrelated to $\z$ (i.e., $f=f_2$), the equality holds, i.e.,
\begin{equation*}
\min_W \E_{\x,\z} \left\| \z - W \cFuntion_2\left(\x\right)  \right\|^2 = \E \| \z \|^2.
\end{equation*}

Therefore, any function $\hatcFuntion$ that minimizes the loss satisfies:
\begin{equation*}
 \min_W \E_{\x,\z} \left\| \z - W \hatcFuntion\left(\x\right)  \right\|^2 = \E \| \z \|^2.
\end{equation*}

We next prove that any $\hatcFuntion$ that minimizes training loss $\mathcal{L}_{\infty, n_2^\prime}$ could not contain any information of $\y$.
Under the condition that $\E \left[\hatcFuntion\left(\x\right) \hatcFuntion^\top\left(\x\right)\right]$ is invertible, we derive that

\begin{equation*}
\begin{split}
    \min_W \E_{\x,\y} \left\| \y - W \hatcFuntion\left(\x\right)  \right\|^2 
    &= \E \|\y\|^2 - \tr \left[\E \left[ \y \hatcFuntion^\top\left(\x\right)\right] \left(\E \left[ \hatcFuntion\left(\x\right) \hatcFuntion^\top\left(\x\right)\right]\right)^{-1} \E \left[ \hatcFuntion\left(\x\right) \y^\top\right]\right]\\
    &= \E \|\y\|^2 - \tr \left[ \bM \E \left[ \z \hatcFuntion^\top\left(\x\right)\right] \left(\E \left[ \hatcFuntion\left(\x\right) \hatcFuntion^\top\left(\x\right)\right]\right)^{-1} \E \left[ \hatcFuntion\left(\x\right) \z^\top\right] \bM^\top \right] \\
    &= \E \|\y\|^2 - \tr \left[ \bM^\top \bM \E \left[ \z \hatcFuntion^\top\left(\x\right)\right] \left(\E \left[ \hatcFuntion\left(\x\right) \hatcFuntion^\top\left(\x\right)\right]\right)^{-1} \E \left[ \hatcFuntion\left(\x\right) \z^\top\right]  \right] \\
    &= \E \|\y\|^2 - \tr \left[\bM^\top \bM\right] \left(\E\|\z \|^2 - \min_W \E\norm{ \z - W \hat{f}\left(\x\right)}^2 \right) \\
    &= \E \|\y\|^2.
\end{split}
    \end{equation*}
However, since there exists $f_1$ such that $\operatorname{Cov}(f_1(\x), \y) \not = 0$, leading to 
\begin{equation*}
 \min_W \E_{\x,\y} \left\| \y - W \hatcFuntion\left(\x\right)  \right\|^2  < \min_W \E_{\x,\y} \left\| \y - W f_1\left(\x\right)  \right\|^2     
\end{equation*}

Therefore, $\hatcFuntion$ violates Criterion~\eqref{C2}.
The proof is done.
\end{proof}

%% file: text/Appendix/pfthm3.tex
\section{Proof of Theorem~\ref{thm: generalloss}}
\label{proof: general}
\generalloss*
\begin{proof}

Denote the trained predictor as $\hatcFuntion$.
Notice that with $n = o\left(\dimc\right)$, $\ySample = \widehat{W}_2 \cFuntion\left(\xSample\right) $ by setting $$\widehat{W}_2 = \left(\yMatrix_{down1}\right)^\top\left[\cFuntion\left(\xMatrix_{down1}\right) \cFuntion\left(\xMatrix_{down1}\right)^\top\right]^{-1} \cFuntion\left(\xMatrix_{down1}\right),$$ where we abuse the notation $f\left(X\right)\in\R^{n\times \dimc}$.
This leads to
$$\min_W \gTwo \rhoTwo\left(\yMatrix_{down1}, W\hatcFuntion\left(\xMatrix_{down1}\right)\right) = 0,$$
on the training set.

Therefore, $\hatcFuntion$ is trained to minimize the loss
$$
\mathcal{L}_1 = -\min_W \E \gOne \rhoOne \left(\z, W \cFuntion\left(\x\right)\right),
$$
leading to the independence between $\cFuntion(\x)$ and $\z$.

By the definition of $\rhoTwo$ and $\rhoOne$, we have for any $\lin>0$,
\begin{equation}
\label{eqn: thm3eq1}
    \begin{split}
\min_{W_1} \E \gTwo \rhoTwo \left(\y, W_1 \cFuntion\left(\x\right)\right) 
\triangleq  \E \gTwo \rhoTwo \left(\y, \widehat{W}_1 \hatcFuntion\left(\x\right)\right)
= \E \gTwo \rhoOne \left(\bar{\y}, \widehat{\overline{W}}_1 \hatcFuntion\left(\x\right)\right),
\end{split}
\end{equation}
where $\bar{\y}, \widehat{\overline{W}}_1$ are the augmented version filled with zero.
By the upper bound of $\frac{ \mathrm{d} \gTwo\left(x\right)}{\mathrm{d} x}$, we have
\begin{equation}
\label{eqn: thm3eq2}
    \begin{split}
        & \gTwo \rhoOne \left(\bar{\y}, \widehat{\overline{W}}_1 \hatcFuntion\left(\x\right)\right) - \gTwo \rhoOne \left(\lin \z, \widehat{\overline{W}}_1 \hatcFuntion\left(\x\right)\right) \\
        =& \frac{ \mathrm{d} \gTwo\left(x\right)}{\mathrm{d} x}|_{x=\xi} \left[\rhoOne \left(\bar{\y}, \widehat{\overline{W}}_1 \hatcFuntion\left(\x\right)\right) - \rhoOne \left(\lin \z, \widehat{\overline{W}}_1 \hatcFuntion\left(\x\right)\right) \right] \\
        \geq & \frac{ \mathrm{d} \gTwo\left(x\right)}{\mathrm{d} x}|_{x=\xi} \left[-\rhoOne\left(\bar{\y}, \lin \z\right)\right]\\
        \geq & -M \rhoOne\left(\bar{\y}, \lin \z\right).
    \end{split}
\end{equation}
where the first equation is due to mean value theorem of integrals with point $\xi$.
The second equality is due to $\frac{ \mathrm{d} \gTwo\left(x\right)}{\mathrm{d} x} \geq 0$ and triangle inequality.
And the third inequality is due to the condition $\frac{ \mathrm{d} \gTwo\left(x\right)}{\mathrm{d} x}|_{x=\xi}\leq M$, and $\xi \leq \max\{ \rhoOne \left(\bar{\y}, \widehat{\overline{W}}_1 \hatcFuntion\left(\x\right)\right), \rhoOne \left(\lin \z, \widehat{\overline{W}}_1 \hatcFuntion\left(\x\right)\right) \}\leq Q_1 + Q_2$.

Besides, notice that
\begin{equation}
\label{eqn: thm3eq3}
    \begin{split}
        & \E \gTwo \rhoOne \left(\lin \z, \widehat{\overline{W}}_1 \hatcFuntion\left(\x\right)\right)  \\
= & \E \gTwo \gOne^{-1} \gOne\left(\rhoOne\left(\lin \z, \widehat{\overline{W}}_1 \hatcFuntion\left(\x\right)\right)\right) \\
\overset{\left(i\right)}{\geq} & \gTwo \gOne^{-1} \E \gOne\left(\rhoOne\left(\lin \z, \widehat{\overline{W}}_1 \hatcFuntion\left(\x\right)\right)\right)  \\
\overset{\left(ii\right)}{\geq} & \gTwo \gOne^{-1} \min_W \E \gOne\left(\rhoOne\left(\lin \z, W \hatcFuntion\left(\x\right)\right)\right) \\
= & \gTwo \gOne^{-1} \max_f \min_W \E \gOne\left(\rhoOne\left(\lin \z, W f\left(\x\right)\right)\right),
    \end{split}
\end{equation}
where inequality~(i) follows Jensen's inequality with condition $\gTwo \gOne^{-1}$ is convex, and inequality~(ii) follows that $\gTwo \gOne^{-1}$ is increasing.
The final equation is due to the independence between $\cFuntion(\x)$ and $\z$.

Combining Equation~\eqref{eqn: thm3eq1}, Equation~\eqref{eqn: thm3eq2}, and Equation~\eqref{eqn: thm3eq3}, leads to
\begin{equation*}
    \begin{split}
&\min_{W_1} \E \gTwo \rhoTwo \left(\y, W_1 \cFuntion\left(\x\right)\right) \\
=& \E \gTwo \rhoOne \left(\bar{\y}, \widehat{\overline{W}}_1 \hatcFuntion\left(\x\right)\right) \\
\geq& \E \gTwo\left(\rhoOne\left(\lin \z, \widehat{\overline{W}}_1 \hatcFuntion\left(\x\right)\right)\right) - M \E \rhoOne \left(\bar{\y}, \lin \z\right)\\
\geq & \gTwo \gOne^{-1} \max_f \min_W \E \gOne\left(\rhoOne\left(\lin \z, W \hatcFuntion\left(\x\right)\right)\right) - M \E \rhoOne \left(\bar{\y}, \lin \z\right)\\
>& \min_f \min_W \E  \gTwo \rhoTwo \left(\y, W_1 \cFuntion\left(\x\right)\right).
    \end{split}
\end{equation*}
This contradicts with Criterion~\eqref{C2}.
The proof is done.
\end{proof}

%% file: text/Appendix/pfthm4.tex
\section{Proof of Theorem~\ref{thm: model}}
\label{proof: model}
\model*

We consider the distribution $\mathcal{P}_0$ as in Section~\ref{Proof: thmrationality}, that is:
Firstly, let $\E[\x]=\E[\y]=\E[\z]=0$ and $\E [\z\z^\top]\! =\! I$.
Given a marginal distribution of $(\x, \y)$, construct a linear relationship $\y = \bM\z$ where the singular values of matrix $\bM$ be $\sigma_1 =  \dots= \sigma_\dimy = \sigma$.
And choose the penalty $\lambda < \sigma^2$.

Firstly, consider the training process.
Under the assumptions on $\cF_1 \times \cF_2$, there exists a ${f}_q\left({f}_2\left(\cdot\right)\right)$ such that:

(a) on the one hand, ${f}_q\left({f}_2\left(\cdot\right)\right)$ minimize $\mathcal{L}_1$.
$${f}_q\left({f}_2\left(\cdot\right)\right) \in \argmin_f \mathcal{L}_1\left(f\right).$$
Note that ${f}_q{f}_2\left(\x\right) \ci \z$ follows ${f}_2\left(\x\right) \ci \z$.
As a result, due to the assumption $\E \z = 0$, we have 
\begin{equation*}
\argmin_W \|\z - W {f}_q{f}_2\left(\x\right)\|=0, \  \min_W \| \z - W {f}_q{f}_2\left(\x\right) \|^2 = \| \z\|^2.
\end{equation*}

And further notice that for any predictor $f$, we have
\begin{equation*}
\min_W \| \z - W {f}\left(\x\right) \|^2 \leq \| \z - 0 {f}\left(\x\right) \|^2 = \|\z \|^2,
\end{equation*}
therefore, ${f}_q\left({f}_2\left(\cdot\right)\right)$ minimize $\mathcal{L}_1$.

(b) on the other hand,  ${f}_q\left({f}_2\left(\cdot\right)\right)$ makes $\mathcal{L}_2$ equal to zero on the training set.
Due to the definition of model capacity, when $\samples < o\left(\mathcal{M}\left(\mathcal{F}_1, \mathcal{L}\right)\right)$, 
by fixing  ${f}_2$, 
there always exists ${f}_q$ such that it can fit any $n$ samples:
\begin{equation*}
   \min_W \norm{\yMatrix_{down1} - W {f}_q\left({f}_2\left(X_{down1}\right)\right)}^2 = 0. 
\end{equation*}

As a result, ${f}_q\left({f}_2\left(\cdot\right)\right)$ can minimize the training loss since it minimize both $\mathcal{L}_1$ and $\mathcal{L}_2$.
Therefore, any predictor $\hat{f}\left(\cdot\right)$ that minimize the training loss must minimize $\mathcal{L}_1$, leading to
\begin{equation*}
\min_W \E \| \z - W \hat{f}\left(\x\right) \|^2 = \E \|\z \|^2.
\end{equation*}
Next we consider the predictor $\hat{f}\left(\cdot\right)$.

\begin{equation*}
\begin{split}
    &\min_W \E_{\x,\y} \left\| \y - W \hatcFuntion\left(\x\right)  \right\|^2 \\
    =& \E \left(\|\y\|^2\right) - \tr \left[ \E \left( \y \hatcFuntion^\top\left(\x\right)\right) \left(\E \left[ \hatcFuntion\left(\x\right) \hatcFuntion^\top\left(\x\right)\right]\right)^{-1} \E \left( \hatcFuntion\left(\x\right) \y^\top\right) \right]\\
    =& \E \left(\|\y\|^2\right) - \tr \left[\bM \E \left( \z \hatcFuntion^\top\left(\x\right)\right) \left(\E \left[ \hatcFuntion\left(\x\right) \hatcFuntion^\top\left(\x\right)\right]\right)^{-1} \E \left( \hatcFuntion\left(\x\right) \z^\top\right) \bM^\top \right]\\
    =& \E \left(\|\y\|^2\right) - \tr \left[\bM^\top \bM \E \left( \z \hatcFuntion^\top\left(\x\right)\right) \left(\E \left[ \hatcFuntion\left(\x\right) \hatcFuntion^\top\left(\x\right)\right]\right)^{-1} \E \left( \hatcFuntion\left(\x\right) \z^\top\right) \right] \\
    =& \E \left(\|\y\|^2\right) - \tr\left[ \bM^\top \bM \left(\min_W \| \z - W \hat{f}\left(\x\right) \|^2 - \|\z \|^2\right) \right]\\
    =& \E \left(\|\y\|^2\right).
\end{split}
    \end{equation*}
    
Notice that $\min_W \E_{\x,\y} \left\| \y - W f_1\left(\x\right)  \right\|^2 < \E \left(\|\y\|^2\right)$ since $\Cov[f_1(\x), \y] \not=0$. Therefore,
\begin{equation*}
 \E_{\x,\y} \left\| \y - W^*_{\y, \hat{f}\left(\x\right)} \hat{f}\left(\x\right)  \right\|^2 > \min_f  \E_{\x,\y} \left\| \y - W^*_{\y, {f}\left(\x\right)} f\left(\x\right)  \right\|^2,
\end{equation*}
which contradicts to the Criterion~\eqref{C2}.

%% file: text/Appendix/example.tex
\section{Illustration of redundant information}

\begin{figure}[t]
    \centering
    \includegraphics[width=0.5 \linewidth]{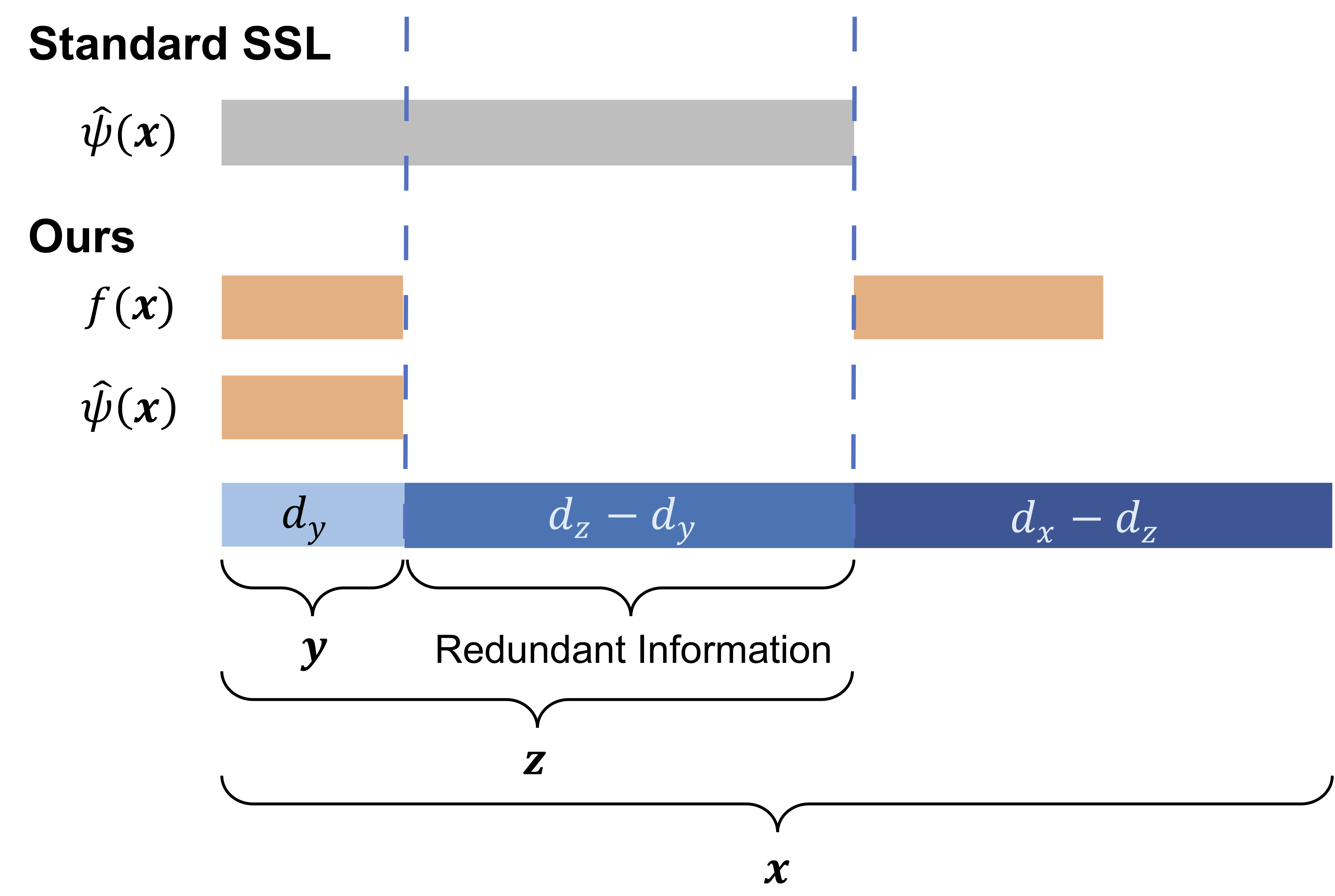}
    \caption{A simple case for understanding redundant information. Here $\y$ is the prefix of $\z$, and $\z$ is the prefix of $\x$. For the standard SSL, the learned representation $\widehat\psi(\x)$ consists of not only the information of $\y$ but also the redundant information in $\z$. For the processor-based SSL, the learned coarse representation $f(\x)$ contains the information of $\y$ and some information of $\x$, and does not include any redundant information in $\z$. Then the final learned representation $\widehat\psi(\x)$ will only contain the information of $\y$, hopefully.}
    \label{fig:denseinfo}
\end{figure}

In the main text, we state that the processor-based SSL aims at eliminating redundant information.
For better understanding, let us consider a simple case.
Suppose $\x$ is a $\dimx$-dimensional random vector.
Let $\z$ and $\y$ consist of the first $\dimz$ and $\dimy$ elements of $\x$, respectively, i.e., $\z=\x[ 0\mycolon\dimz-1]$ and $\y=\x[0\mycolon\dimy-1]$.
When we execute the standard self-supervised learning (with linear representation $\hat{\psi}$) using infinite pretext samples,
the learned representation has rank $\dimz$ which consists of $\dimy$-dimensional useful information and $(\dimz-\dimy)$-dimensional redundancy information that we cannot eliminate (See Figure~\ref{fig:denseinfo}).
Therefore, we need $\tilde{\O}(\dimz)$ samples to train the liner layer $\widehat{W}$ in the downstream tasks.
In comparison, the processor approach requires a function $\cFuntion$ that eliminates the redundancy information at the beginning, i.e., $\cFuntion(\x) = \x[ 0\mycolon\dimy-1, \dimz\mycolon\dimx]$.
Hopefully, $\cFuntion(\x)$ will consist of all the useful information and some information of $\x[\dimz\mycolon\dimx-1]$, without any information of $\x[\dimy\mycolon\dimz-1]$.
And then we apply pretext task on ($\cFuntion(\x), \z$), which further removes the information of $\x[\dimz\mycolon\dimx-1]$.
Thus, the rank of representation $\widehat\psi(\x)$ could be reduced to $\dimy$. 
In this way, we will only need $\tilde{\O}(\dimy)$ samples for the downstream task.

%% file: text/Appendix/AppendixExperiment1.tex
\section{Experimental details}

\subsection{Experimental details for synthetic data}
\textbf{Datasets and environment.}
The dataset is simulated as follows.
Firstly, we randomly generate $n_1$ pretext samples $\{(x^i_{pre}, z^i_{pre})\}_{i\in [n_1]}$, where $x^i_{pre} \in \R^\dimx$ is generated from Gaussian distribution $\mathcal{N}(0,I_{d_x})$, and $z^i_{pre} \in \R^\dimz$ is a perturbation of the first $\dimz$ elements of $x^i_{pre}$, namely, $z^i_{pre} = x^i_{pre}[0:\dimz] + 0.01\cdot  \varepsilon_1^i$, where $\varepsilon_1^i\sim \mathcal{N}(0,I_{d_z})$.
For the ease of notations, we present the data using sample form instead of the matrix form ($X_{pre}$, $Y_{pre}$).

Similarly, we generate $n_2$ downstream samples $\{(x^i_{down}, y^i_{down})\}_{i\in[n_2]}$, where $x^i_{down} \in \R^\dimx$ is generated from Gaussian distribution $\mathcal{N}(0,I_{d_x})$ and $y^i_{down}\in\R^\dimy$ is a perturbation of the first $\dimy$ elements of $x^i_{down}$, i.e., $y^i_{down}= x^i_{down}[0\mycolon\dimy] + 0.01 \cdot \varepsilon_2^i$, where $\varepsilon_2^i\sim \mathcal{N}(0,I_{d_y})$.
For the processor training procedure, we use $n_0$ of downstream data as $\{(x^i_{down1}, y^i_{down1})\}_{i\in[n_0]}$, and the rest $n_2 - n_0$ samples are used in downstream tasks (denoted by $\{(x^i_{down2}, y^i_{down2})\}_{i\in[n_2-n_0]}$).
We require $n_0 = \alpha n_2$, and setting $\alpha = 0.5$ without loss of generality.
We generate $n_t$ test samples under the same procedure.

In each run of our proposed algorithm, we first use $n_0$ samples $\{(x^i_{down1}, y^i_{down1})\}_{i\in[n_0]}$ and $n_1$ pretext samples to train the processor.
Then we apply the processor and the $n_1$ pretext samples $n_2-n_0$ downstream samples to finish the training.
For a fair comparison, we use $n_1$ pretext samples to train the feature during the SSL procedure and then use $n_2$ downstream samples to learn the linear layer.
We run each experiment 5 times repeatedly and calculate its mean and standard deviation. 
We plot their 95\% confidence bands in all the figures.

\textbf{The rationality of the loss $\mathcal{L}(f)$.} 
Set $\dimx = 100$, $\dimz=80$, $\dimy=5$, $\dimc=5$, $n_1 = 20000$ and $\lambda = 0.1$.
We plot the MSE as $n_0$ varies from 15 to 70 in Figure~\ref{fig: changenxy}.
With large $n_0$, the downstream tasks benefit from small MSE, showing that the algorithm indeed finds the proper coarse representations.
Theorem~\ref{thm: rationality} demonstrates this phenomenon, which claims that with sufficient downstream task samples, the proposed algorithm finds the proper coarse representations.

We further remark that in the case $n_0 = 40$, the SSL suffers from a large MSE. 
It is related to the ``double decent'' phenomenon in the downstream tasks, and MSE reaches the maximum when $n_0 = \dimz$.
We finally remark that when $n_0 \in [20, 40]$, the proposed algorithm outperforms SSL in terms of MSE.

\textbf{Large coefficient $\lambda$ harms the performance.}
Set $\dimx = 100$, $\dimz=80$, $\dimy=5$, $\dimc=5$, $n_1 = 20000$ and $n_2 = n_0 = 15$.
We test the case $\lambda$ varies from $0.1$ to $1.5$, as plotted in Figure~\ref{fig: changelamb}.
Figure~\ref{fig: changelamb} illustrates that when $\lambda$ is large, the algorithm suffers from unsatisfying performances.
Theorem~\ref{thm: rationality} demonstrates the phenomenon.
Intuitively, that is because $\z$ contains the information of $\y$.
With a large coefficient $\lambda$, the trained classifier tends to eliminate the information of $y$, leading to a lousy model performance (MSE).

\textbf{Large dimension $\dimc$ harms the performance.}
Set $\dimx = 100$, $\dimz=80$, $\dimy=5$, $n_1 = 20000$, $n_2 = n_0 = 15$, and $\lambda=0.1$.
We test the case $\dimc$ varies from $1$ to $12$, and plot them in Figure~\ref{fig: changec}.
Firstly, notice that MSE reaches the minimum when $\dimc = \dimy$, that is because we force the coarse representation to have no redundancy information by limiting its dimension.
Secondly, notice that the figure shows a Basin-type phenomenon and reaches the minimum when $\dimc = \dimy$.
On the one hand, when $\dimc < \dimy$, the coarse representation is forced to lose the information of $y$, leading to the underfitting phenomenon.
On the other hand, when $\dimc > \dimy$, the performance of coarse representation is limited by the downstream sample numbers, as illustrated in Theorem~\ref{thm: warm-up} and Theorem~\ref{thm: model}.

%% file: text/Appendix/appendix.tex
\subsection{Experimental details for real-world data}
All experiments can be handled with an RTX 2080Ti graphic card. 
To simplify the comparison, we require that the processor-training method uses the same downstream samples in downstream tasks as standard SSL.
Therefore, the processor training process is extra compared to standard SSL.
We emphasize that the processor training method in extra-sample settings performs better than the settings in the main text (which split some downstream samples to train the processor), and the negative results in such extra-sample settings are stronger.

\textbf{Data split strategy.}
We first classify the samples according to their categories and sort them according to their index in the CIFAR-10 dataset. And then pick the correct number of samples in each class one by one without overlapping. For each subset, we assure that there is the same number of samples for each class. 

\textbf{Model training.}
For standard SSL training, we directly follow all the hyperparameters and training strategies from~\cite{rotation} except for the training set size. 

For processor-based learning, we also copy the hyperparameters from the standard setting, which is necessary. There are two additional SGD optimizers for training $W^*_{\y, \cFuntion(\x)}$ and $W^*_{\z, \cFuntion(\x)}$, which are two linear layers with the same input size and different output size in our implementation. We set the learning rate as 0.1 and train 20 times before calculating the loss function used to update $f$. When optimizing Equation~\eqref{eqn: lossterm}, $\lambda$ is set to 0.1 to balance the two loss terms in Equation~\eqref{eqn: lossterm}.

For the half NIN model, we half the channel number in each block. Consequently, the size of all linear layers should also be changed.